\documentclass{article}
\PassOptionsToPackage{dvipsnames}{xcolor}

\usepackage{iclr2027_conference,times}
\usepackage[T1]{fontenc}
\usepackage[utf8]{inputenc}
\DeclareUnicodeCharacter{266B}{\textmusicalnote}
\usepackage{microtype}
\usepackage{xspace}

\usepackage{amsmath, amssymb, amsthm}
\usepackage{mathtools}
\usepackage{bm}

\usepackage{booktabs}
\usepackage{enumitem}
\usepackage{xcolor}
\definecolor{darkgreen}{rgb}{0,0.5,0}
\definecolor{gaingreen}{rgb}{0.00,0.62,0.22}
\definecolor{evopink}{HTML}{D6337F}
\usepackage{multirow}
\usepackage{colortbl}
\usepackage{adjustbox}
\usepackage{graphicx}
\usepackage{wrapfig}
\newcommand{\std}[1]{{\scriptsize$_{\pm#1}$}}
\newsavebox{\objtabbox}

\usepackage{url}
\usepackage[hidelinks]{hyperref}
\usepackage{cleveref}
\crefname{assumption}{Assumption}{Assumptions}
\usepackage{needspace}

\newcommand{\chg}[1]{}

\newcommand{\E}{\mathbb{E}}
\newcommand{\Prob}{\mathbb{P}}

\newcommand{\Rt}{R_{\eta}}
\newcommand{\Gph}[1]{\mathcal{G}^{(#1)}}
\newcommand{\Slib}[1]{\mathcal{S}^{(#1)}}
\newcommand{\PB}{P_{B}}
\newcommand{\PBz}{P_{B}^{0}}
\newcommand{\PF}{P_{F}}
\newcommand{\Fpsi}{F_{\psi}}
\newcommand{\dbres}{\delta}
\newcommand{\ex}{\mathrm{exec}}
\newcommand{\eps}{\epsilon}
\newcommand{\LTB}{\mathcal{L}_{\mathrm{flow}}}
\newcommand{\RtwoFlow}{R\textsuperscript{2} Flow\xspace}
\newcommand{\RtwoTB}{R\textsuperscript{2}TB\xspace}
\newcommand{\Var}{\operatorname{Var}}
\DeclareMathOperator{\Pa}{Pa}
\DeclareMathOperator{\Inn}{In}
\DeclareMathOperator{\Out}{Out}

\theoremstyle{definition}
\newtheorem{definition}{Definition}
\newtheorem{assumption}{Assumption}
\theoremstyle{plain}
\newtheorem{proposition}{Proposition}
\newtheorem{corollary}[proposition]{Corollary}
\newtheorem{lemma}[proposition]{Lemma}
\theoremstyle{remark}
\newtheorem{remark}{Remark}

\title{\RtwoFlow: Recursive Self-Improvement via\\
       Recursive Skill Evolution\vspace{-8pt}}

\author{Mingda Zhang$^{1}$,\ \ Qiang Huang$^{2}$,\ \ Yanjin Li$^{3}$,\ \ Zijia Wang$^{4}$, \\
\bfseries Qika Lin$^{5}$,\ \ Xiaoying Tang$^{1}$,\ \ Tiesunlong Shen$^{5}$ \\
{\small\mdseries $^{1}$The Chinese University of Hong Kong, Shenzhen\quad $^{2}$Fudan University} \\
{\small\mdseries $^{3}$University of Illinois at Urbana-Champaign\quad $^{4}$University of Oxford\quad $^{5}$National University of Singapore} \\[3pt]
{\small\mdseries\texttt{mingdazhang@ieee.org}}
}

\iclrfinalcopy
\makeatletter
\def\@maketitle{\vbox{\hsize\textwidth\centering
{\LARGE\sc \@title\par}
\lhead{Under review as a conference paper at ICLR 2027}
\begin{tabular}[t]{c}\bf\rule{\z@}{24pt}\@author\end{tabular}%
\vskip 0.1in minus 0.1in}}
\makeatother

\begin{document}
\maketitle
\lhead{Under review as a conference paper at ICLR 2027}
\vspace{-14.4pt}

\begin{abstract}
LLM-based agents can improve themselves across tasks by reusing and revising the skills they orchestrate into executable procedures. Flow-based training fits this loop: it samples procedures in proportion to reward, and the flow through each skill credits it for the next library revision. Three obstacles stand in the way of making this self-improvement reliable: flow training suffers strategy collapse over tree-structured histories; nonnegative flow-based credit rewards frequent use as if it were benefit; and library edits rest on the task reward the policy optimizes. We introduce \textbf{\RtwoFlow}, a recursive self-improvement framework that alternates policy learning, independent verification, and
versioned skill-library updates on a shared-state orchestration graph. The graph
merges histories that differ only in the order of independent steps, allowing
flow training to pool evidence across equivalent executions. A flow-share readout of the trained flow, invariant to the backward policy, and a separate signed utility rank which skills to change, verifier evidence decides
whether an edit is warranted, and a residual-variance plateau sets when to update. Committed edits reshape the graph the next policy learns on, realizing recursive skill evolution. Across question answering, mathematical reasoning,
interactive decision making, and code generation, \RtwoFlow improves task
accuracy and library-edit precision over heuristic orchestration, reinforcement
learning, and skill-evolution baselines, and transfers across executors. Code is available at \url{https://github.com/beita6969/r2flow}.
\end{abstract}

\section{Introduction}

\begin{wrapfigure}[16]{r}{0.54\textwidth}
\setlength{\abovecaptionskip}{0pt}
\vspace{-10mm}
\centering
\includegraphics[width=7.45cm]{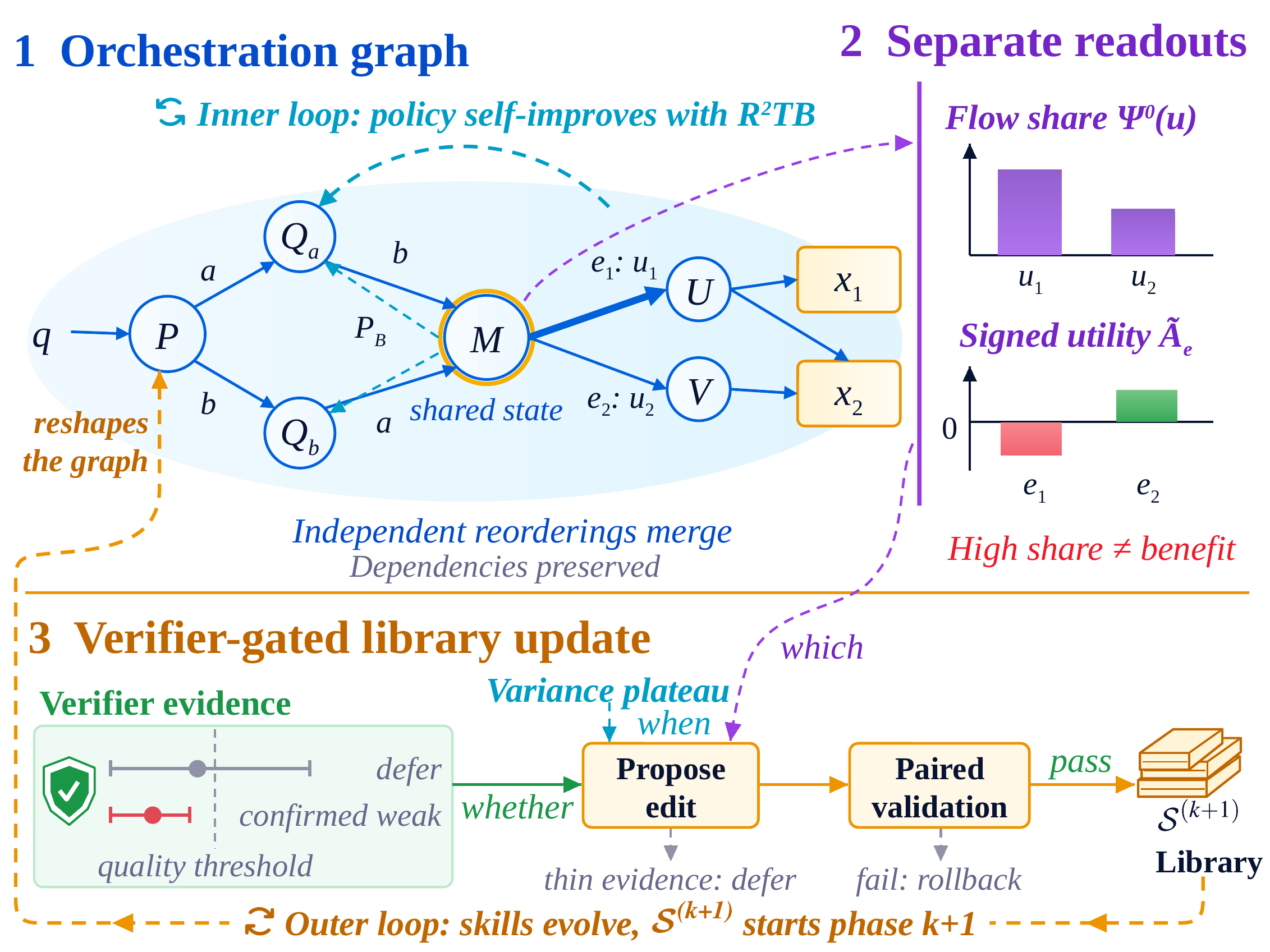}
\caption{\RtwoFlow trains a flow on shared states (inner loop), reads share and signed utility separately, and gates library edits on verified evidence (outer loop).}
\label{fig:overview}
\vspace{-8mm}
\end{wrapfigure}
In recent years, LLM-based agentic systems have come to orchestrate reusable
skills into executable procedures
\citep{yao2022react,wang2023voyager,yue2026static}, shifting the unit of
automation from a single answer to an entire procedure.
A central goal is therefore not only to complete a task but to improve the
procedure used later
\citep{fang2025comprehensive,ren2026self,chen2026recursive}.
This makes self-improvement a recursive loop: each phase trains the policy on a versioned library and updates that library from what it learned, and the next phase repeats this on the new library. Without human review, such a loop must admit each update on evidence outside the training reward.

Task orchestration organizes primitive actions and reusable skills into
structured trajectories on an orchestration graph (\Cref{fig:overview},
panel~1), where each node is an execution state and each edge is one
orchestration event, which makes complex tasks more controllable and composable
\citep{zhuge2024language,zhang2025aflow,yun2026graph,tieu2026inference}.
Orchestration has accordingly evolved from static human-authored pipelines to
graphs generated or edited at run time
\citep{zhang2025aflow,zhang2026flowsteer,yue2026static}, and further to
policies trained directly from execution feedback
\citep{zhang2026skillflow,zhang2026reinforcement}.
However, these policies are still trained on tree-structured histories, where the
same steps taken in a different order are never recognised as one procedure.

\begin{figure}[t]
\centering
\includegraphics[width=\linewidth]{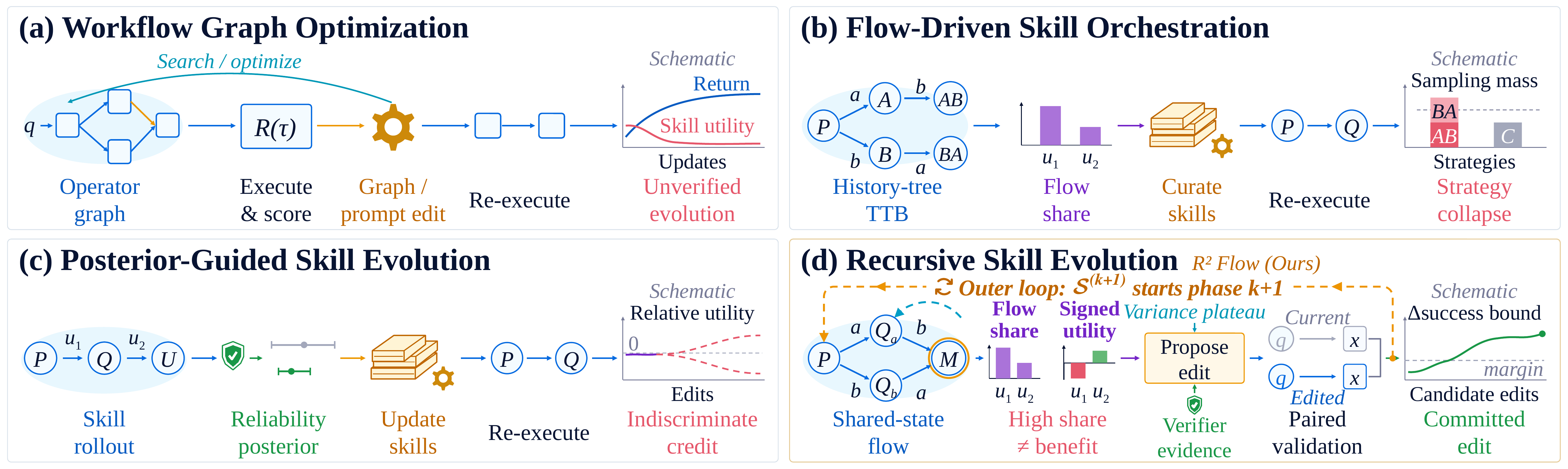}
\caption{Four orchestration paradigms. (a) Workflow graph optimization hides a
harmful skill in the return. (b) Flow-driven orchestration learns $AB$ and $BA$
twice on a history tree, so evidence never pools. (c) Posterior-guided evolution scores reliability, not signed utility, so success reads as benefit. (d) \RtwoFlow pools evidence on shared states, separates share from utility,
and commits verified edits.}
\label{fig:paradigms}
\vspace{-3pt}
\end{figure}

To address these issues, \emph{graph-based orchestration} (\Cref{fig:paradigms}a) retrieves prebuilt
skills or searches and edits workflow graphs at run time
\citep{zhang2025aflow,zhuge2024language,yun2026graph,li2026graphflow,tieu2026inference}.
\emph{Step-level credit assignment} (\Cref{fig:paradigms}b) converts a terminal reward into finer-grained
signals through process reward models, hindsight advantage estimation, and GFlowNet balance objectives that sample in proportion to reward \citep{bengio2021flow,malkin2022trajectory,zhang2026skillflow,tao2026trace,tan2026hindsight}.
\emph{Dynamic skill libraries} (\Cref{fig:paradigms}c) grow capability by curating skills with
LLM-as-judge prompting, execution feedback, or posterior beliefs over skill
reliability
\citep{wang2023voyager,wu2026bayesian,tang2026wikiskill,ouyang2026skillos}.

However, joining these lines of work into one loop faces three challenges.
\textbf{(i)~Strategy collapse.} When a state is the full interaction history,
every state has exactly one parent, so the orchestration graph is a tree and the backward policy is trivially one \citep{malkin2022trajectory,fawkes2026f,wang2026rooted}.
Executions that differ only in the order of independent steps sit on different
branches and never share evidence; the flow then spreads probability over reorderings of one procedure rather than across genuinely different strategies, so a procedure with more admissible orderings receives proportionally more sampling mass, and more weight in the next library edit \citep{bengio2021flow,zhang2026skillflow}.
\textbf{(ii)~Indiscriminate credit assignment.} Flow shares and usage counts are non-negative by construction and reliability posteriors score whether a call succeeds, so none
says whether a skill helps, and a frequently invoked harmful skill looks critical and survives into the next library \citep{zhang2026skillflow,wu2026bayesian,li2026dynamic}.
\textbf{(iii)~Unverified skill evolution.} Edits are admitted on the
same task reward whose attribution is in question, and on point estimates that
cannot separate insufficient evidence from confirmed failure, so a skill can be pruned before its evidence accumulates \citep{zhang2025aflow,wu2026bayesian,zhang2026library,zhang2026credit,zhang2026reasoning}.

We introduce \textbf{\RtwoFlow} (\Cref{fig:paradigms}d), running this loop through three designs, each serving the library update.
The first, the \textbf{Orchestration Graph}, replaces the history tree of trained orchestration policies: it merges histories differing only in the order of independent steps into one state, so their evidence pools at that state; when a merged state has
several legal predecessors, the backward policy is a genuine choice among
them.
Built on the graph, we further propose \textbf{flow share}: a flow network is trained on the shared states with the balance objective \RtwoTB, and each skill's share of the flow is read from the trained policy's rollouts weighted by their residuals, a readout invariant to the backward policy. In contrast to credit that reads use as benefit, it is kept separate from a signed utility that says whether the skill helps.
\RtwoFlow then performs \textbf{verifier-gated skill-library updates}, replacing admission on the task reward:
independent verifiers label skill invocations, edge flow picks which ones to verify, and an edit refines or prunes a skill only when verified evidence
supports it, waiting otherwise. The \RtwoTB residual also times the loop: a plateau in its variance ends a phase and triggers the next library update. Together the designs close one recursive loop: each phase trains on the library the previous phase committed, and its readouts, posterior and residual decide the next commit.

We evaluate on twelve datasets across four task domains, six held out, reporting task accuracy, executor
transfer, and library-edit precision against heuristic orchestration,
reinforcement learning \citep{zhang2026flowsteer}, and skill-evolution baselines \citep{ni2026trace2skill,yang2026skillopt}.

\section{Related Work}

\paragraph{Self-Improving Agents.}
Self-improving agents rewrite scaffolding, meta-agent designs, or their own
logic and code
\citep{zelikman2023self,hu2025automated,yin2025godel,zhang2026darwin}, in loops
framed as bounded or open-ended self-modification
\citep{chen2026recursive}.
Skill evolution co-trains the library with the policy \citep{xia2026skillrl,zhang2026skillflow}, Bayesian curation adds finite-sample uncertainty, and verification screens edits \citep{guo2026skill,wu2026bayesian,chen2026skillcat,zhang2026coevoskills}.
Most admit an edit on task-level outcomes, so the edit is judged
by the signal whose attribution is in question, where an optimiser exploits its
evaluator \citep{thaman2026reward}.
\RtwoFlow instead admits library edits on verifier labels outside the reward.

\paragraph{Agent Task Orchestration.}
LLM agents orchestrate tasks in loops that plan, act, and revise on execution
feedback \citep{yao2022react,yue2026static}.
Existing work spans operator graphs generated by search or edited under
structured feedback
\citep{zhang2025aflow,zhuge2024language,zhang2026flowsteer,li2026graphflow,tieu2026inference},
reusable skill memory, and self-evolving architectures
\citep{wang2023voyager,tang2026wikiskill,xu2026hyperskill,bai2026skilldag,lin2026muse}.
Trained orchestration policies define states over raw histories, so a procedure and its reorderings never share statistics.
\RtwoFlow instead trains its flow on a dependency-aware shared-state graph of executable histories, where reorderings meet at one state \citep[as in reasoning DAGs;][]{yu2024flow}.

\section{Preliminaries}

\noindent\textbf{Definition 1: Orchestration State Graph.}
Task orchestration is a directed graph $\mathcal G$ whose vertices are interaction histories $H_t$ and whose edges are orchestration events.
The update $H_t=H_{t-1}\oplus(r_t,e_t,o_t^{\ex})$ appends one record, so $\mathcal G$ is acyclic and every non-root history has one parent.

\noindent\textbf{Definition 2: Orchestration Trajectory.}
A complete path from the initial history $H_0$ to a terminal state defines an orchestration trajectory:
\begin{equation}
  \tau = \bigl\{(r_t,\ e_t,\ o_t^{\ex})\bigr\}_{t=1}^{T}\ \Rightarrow\ y_q ,
\end{equation}
where $r_t$ is the reasoning text, $e_t=(u_t,\mathrm{args}_t)$ the
event, and $o_t^{\ex}$ the execution feedback. It ends
on acceptance or budget exhaustion; its terminal outcome $x$ earns a reward
$R(x)\in[0,1]$.

\noindent\textbf{Problem Statement.}
Given a task $q$ and an environment $\mathcal{E}$ with a frozen executor
$\mathcal{M}_{\ex}$, we augment the graph with a non-negative edge flow
\citep{bengio2021flow,bengio2023gflownet} that is conserved at every non-terminal
state, with terminal condition $F(x)=\Rt(x)$ for the tempered reward
$\Rt(x)=(R(x)+\epsilon)^{\eta}$.
The forward and backward policies are the normalised edge flows,
$\PF(e\mid s)=F(s\!\to\! e)/F(s)$ and $\PB((s,e)\mid s')=F(s\!\to\! e)/F(s')$, and a
flow meeting the terminal condition is called \emph{reward-matching}.
We seek a state map $\sigma$ that merges reorderings of independent steps, preserves
legal events, successors and terminal reward, and induces a graph admitting a reward-matching
flow with terminal marginal proportional to the tempered reward:
\begin{equation}
  \PF(x\mid q) \;=\; \frac{\Rt(x)}{Z^\star(q)},
  \qquad
  Z^\star(q)=\sum_{x\in\mathcal X}\Rt(x),
\end{equation}
where $Z^\star(q)$ is the true partition function.
The flow must also give a per-skill attribution well defined on that
graph, to rank the edits turning the phase-$k$ library $\Slib{k}$ into $\Slib{k+1}$ (Appendices~A and~D).

\section{Methodology: \RtwoFlow}
\label{sec:method-anchor}
\begin{figure}[t]
\centering
\includegraphics[width=\linewidth]{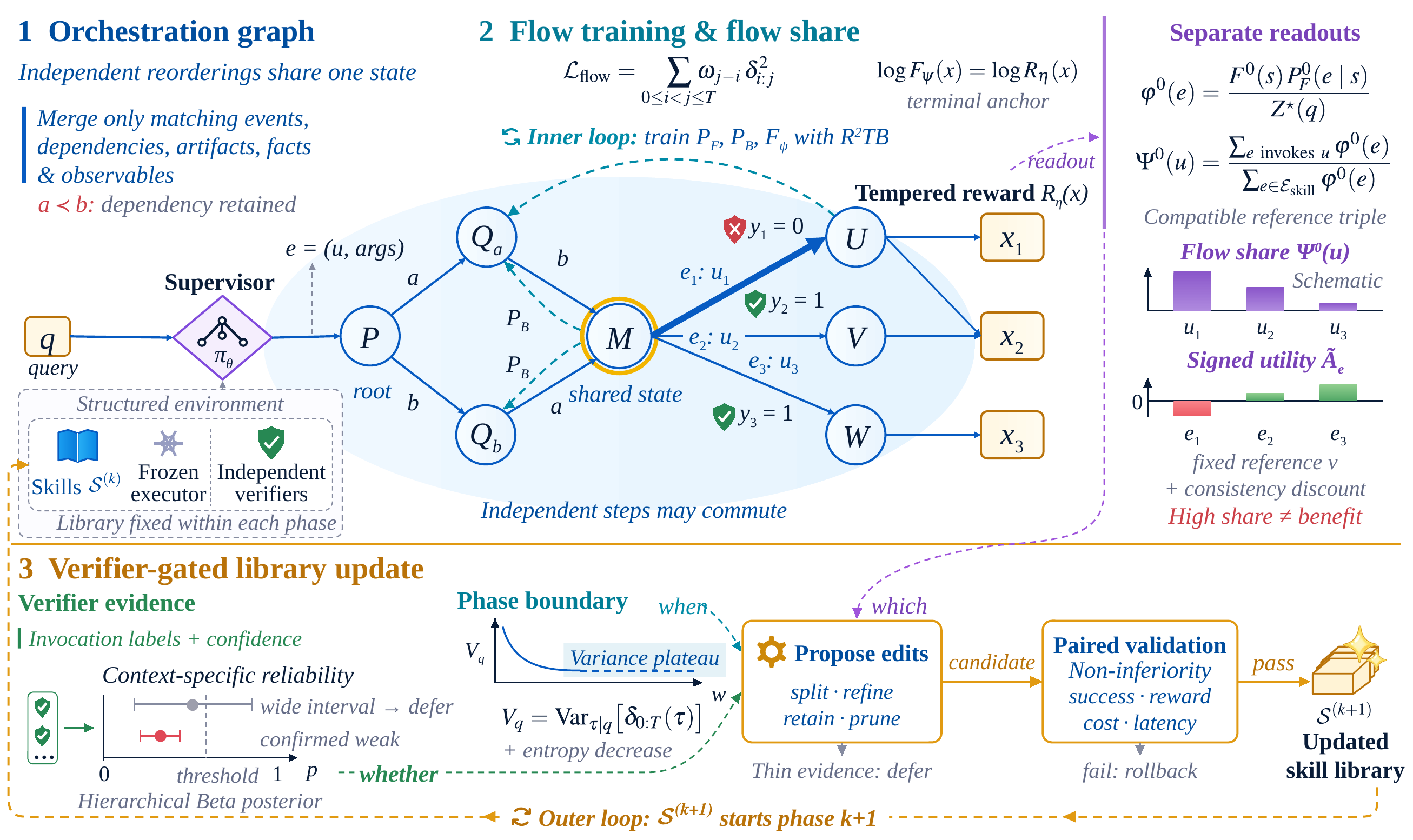}
\caption{The \RtwoFlow framework. (1) The Supervisor $\pi_\theta$ acts in a
structured environment where histories differing only in the order of
independent steps share one state. (2) $\PF$, $\PB$ and $\Fpsi$ are
trained with \RtwoTB; flow share $\Psi^0(u)$ and signed utility
$\widetilde{A}_e$ are read separately. (3) Verifier evidence, the two readouts
and the variance plateau decide whether, which and when to edit.}
\label{fig:framework}
\end{figure}

In \Cref{fig:framework}, \RtwoFlow nests two self-improvement loops: within a phase (inner loop), the policy improves on its own rollouts with \RtwoTB until the residual variance plateaus; across phases (outer loop), verified edits update the library, and with it the next phase's graph. The graph (Section~4.1) carries a flow trained with \RtwoTB (Section~4.2) whose readouts and residual steer the edits (Section~4.3).

\subsection{Orchestration Graph and Task Modeling}
\RtwoFlow follows a Supervisor--Executor
paradigm \citep{yao2022react}: a trainable Supervisor $\pi_\theta$ interacts with
a structured environment $\mathcal{E}$ whose states are \textbf{shared states} rather than raw histories.

\noindent\textbf{Structured Environment} $\mathcal{E}$.
It holds the skill library, a verifier suite, and the executor:
\begin{equation}
  \mathcal{E} = (\mathcal{S},\ \mathcal{Y},\ \mathcal{M}_{\ex}),
\end{equation}
where $\mathcal{S}$ is the dynamic skill library, $\mathcal{Y}$ a suite of
independent verifiers labelling single skill invocations, and
$\mathcal{M}_{\ex}$ the pluggable executor. The library changes at phase boundaries, so the action space is fixed within a phase and every trajectory in
that phase is sampled against one library version.

\noindent\textbf{Shared States and the Merge Rule.}
A raw history $H_t$ (Definition~1) maps to its shared state by $\sigma$:
\begin{equation}
  s_t = \sigma(H_t) = \bigl(E_t,\ \prec_t,\ \mathcal{A}_t,\ \mathcal{V}_t,\ \Omega_t\bigr),
\end{equation}
where $E_t$ is the set of committed events, $\prec_t$ their dependency order,
$\mathcal{A}_t$ the canonicalised artifacts, $\mathcal{V}_t$ the verified facts,
and $\Omega_t$ the environment observables.
For example, independent checks $A$ and $B$ can run in either order, so
$A\!\rightarrow\!B$ and $B\!\rightarrow\!A$ map to one state; if $B$ consumes
$A$'s output, the dependency $A\prec B$ is retained and the orders stay
distinct (legality conditions on $\sigma$ in Appendix~A).

\noindent\textbf{In-Edges and Rank.}
A merged state has several in-edges, each from removing a $\prec$-maximal event:
\begin{equation}
  \Inn(s') = \bigl\{(s,e) : s' = s\oplus e\bigr\},
  \qquad
  \rho(s) = |E_s|,
\end{equation}
where $\rho$ is a state's rank; each edge commits an event, so $\rho$ increases
along paths and, with the step budget, the graph is a finite DAG. Each predecessor of $s'$
is an in-edge for the backward policy.

\noindent\textbf{Terminal States and Tempered Reward.}
The terminal state $x$
carries the tempered reward $\Rt(x)$ of Section~3, whose
exponent $\eta$ sets how sharply sampling favours high reward and whose
$\epsilon>0$ keeps it positive; $R$ is a function of the terminal state, so histories merged into one state share one reward.

\medskip\noindent\textbf{Proposition 1.} \emph{Merging reorderings yields a shared-state DAG; whenever a reachable state has multiple legal predecessors, the backward policy is a genuine choice among them.} (Empirically, the shared-states and backward-policy ablations and strategy counts, RQ4 in \S\ref{sec:setup}; formally, Appendix~A.)

\subsection{Flow Training and Flow Share}
As shown in \Cref{fig:framework}, panel~2, over the shared-state graph from Section~4.1,
\RtwoFlow trains a flow network $(P_{F,\theta}, P_{B,\phi}, F_\psi)$ targeting the reward-matching flow of Section~3, whose forward policy $P_{F,\theta}$ is the Supervisor $\pi_\theta$; throughout, $\PF(e\mid s)$ denotes the one-step edge policy.
Beyond mode coverage, the flow serves the edit step: $F(s)$ is shared by every path through $s$, so where paths merge, the reference flow $F^0$ gives the readout of Section~4.3, alongside the signed utility below.

\noindent\textbf{Forward Policy $\PF$ and Backward Policy $\PB$.}
An event is a skill identifier with structured arguments, so the forward policy
factorises, while the backward policy is normalised over in-edges:
\begin{equation}
  \PF(e\mid s) = \PF(u\mid s)\,\PF(\mathrm{args}\mid u,s),
  \qquad
  \PB\bigl((s,e)\mid s'\bigr)
  = \frac{\exp f_\phi(s,e,s')}{\sum_{(\bar s,\bar e)\in\Inn(s')}\exp f_\phi(\bar s,\bar e,s')},
\end{equation}
where $f_\phi$ scores an in-edge from the hindsight state. Each in-edge removes a different $\prec$-maximal event, so $\PB$ spreads training credit over the skills that reach $s'$. On a history tree every state has one in-edge, so $\PB\equiv1$ and this credit is never split (Appendix~A).

\noindent\textbf{Recursive-Residual Trajectory Balance (\RtwoTB).}
Given $\PF$, $\PB$, a log-flow head $\log\Fpsi$ and the reasoning text $r_t$, a sub-trajectory from $s_i$ to $s_j$ has the balance residual (SubTB;
\citealp{madan2023learning}):
\begin{equation}
  \delta_{i:j} = \ell(s_i) + \sum_{t=i+1}^{j}\log\PF(e_t\mid s_{t-1},r_t)
               - \ell(s_j) - \sum_{t=i+1}^{j}\log\PB\bigl((s_{t-1},e_t)\mid s_t\bigr).
\end{equation}
The training objective is the weighted squared residual, with the terminal flow
anchored to the reward:
\begin{equation}
  \mathcal{L}_{\mathrm{flow}} = \sum_{0\le i<j\le T}\omega_{j-i}\,\delta_{i:j}^{2},
  \qquad
  \ell(s)=\log\Fpsi(s)+b_d\;\,(s\notin\mathcal X),\quad \ell(x)=\log\Rt(x),
\end{equation}
where $\Fpsi(s_0)=Z_\theta(q)$ is the learned partition function, $\omega_{j-i}$ decays geometrically in $j-i$ (as in SubTB) and is normalised per trajectory, and the
per-domain bias $b_d$ absorbs a shared offset of the terminal residuals (Appendix~B). Here $j=i+1$ recovers detailed balance and $(i,j)=(0,T)$ trajectory balance, so
one loss interpolates between local bootstrapping and trajectory variance, and
the target terminal law holds at zero loss (\cref{prop:subtb-exact,cor:terminal-law}). \RtwoTB is sub-trajectory balance on the shared-state graph with the per-domain bias $b_d$ and the reasoning-conditioned forward term, and its residual variance sets the phase boundary of Section~4.3.

\noindent\textbf{Flow Share and Signed Utility.}
Flow share measures how much of the reward-matching flow passes through each skill: at exact balance, a compatible reference triple $(F^0,\PF^0,\PBz)$ makes the normalised edge flow the probability of traversing that edge, and a skill's share sums them over its edges:
\begin{equation}
  \varphi^0(e) = \frac{F^0(s)\,\PF^0(e\mid s)}{Z^\star(q)},
  \qquad
  \Psi^0(u) = \frac{\sum_{e\ \text{invokes}\ u}\varphi^0(e)}
                   {\sum_{e\in\mathcal E_{\mathrm{skill}}}\varphi^0(e)},
\end{equation}
where $F^0$ and $\PF^0$ are defined by the uniform reference backward
policy $\PBz$ (Appendix~D), which also fixes the edge share $\varphi^0(e)$ used to allocate verification.
The denominator sums over skill edges, so the shares of all skills sum to one. Every path to a terminal commits the same events, so $\Psi^0(u)$ is the share of $u$ among skill calls at terminals drawn in proportion to $\Rt$, for any backward policy (\Cref{cor:share-estimator}). It is estimated per query from trained-policy rollouts, each weighted by $e^{-\delta_{0:T}}$ normalised within the query, and averaged over queries; the normalisation cancels $\log Z_\theta(q)$ and $b_d$. Flow share is non-negative by construction and measures how much flow a skill carries; to measure whether it helps, following step-level and counterfactual credit \citep{zhang2026trca,lu2026agents,li2026counterfactual}, we add a signed utility $\widetilde{A}_e$ contrasted with a fixed reference policy $\nu$:
\begin{equation}
  \widetilde{A}_e = \Bigl(\widehat{\E}[\Rt\mid s,e] - \E_{e'\sim\nu}\bigl[\widehat{\E}[\Rt\mid s,e']\bigr]\Bigr)
                    \exp\!\bigl(-|\delta_e|/\tau_c\bigr),
\end{equation}
where $\widehat{\E}[\Rt\mid s,e]$ estimates the expected tempered reward after taking $e$ at $s$ under a fixed continuation policy (\cref{eq:util}), $\nu(\cdot\mid s)$ is uniform over the legal events at $s$, $\delta_e = \delta_{t-1:t}$ is the single-edge balance violation, and $\tau_c > 0$ scales a discount that weights each estimate by how well the flow balances on that edge. A positive $\widetilde{A}_e$ means the event beats an average legal alternative at $s$. In \Cref{fig:framework}, skill $u_1$ carries the most flow while its edge $e_1$ has negative signed utility.

\medskip\noindent\textbf{Proposition 2.} \emph{At zero \RtwoTB loss the Supervisor samples
terminals in proportion to the tempered reward, and flow share is conserved,
independent of the backward policy, consistently estimated from the rollouts of any trained flow, and separate from whether a skill helps.} (Empirically, the reference-readout and utility ablations, RQ4 in \S\ref{sec:setup}; formally, Appendices~A, B and D.)

\subsection{Recursive Skill Evolution with Verifier-Gated Updates}
Section~4.2 provides two readouts, the share $\Psi^0(u)$ and
the signed utility $\widetilde{A}_e$; both are functions of the reward the
library is judged by, so an edit's evidence must come from the verifiers of
Section~4.1 rather than that reward. This section (\Cref{fig:framework}, panel~3) combines them with
verifier evidence once per phase: verifier labels answer
\emph{whether}, the readouts answer \emph{which}, and the residual
answers \emph{when}.

Each phase applies one fixed update, repeated until the compute budget is spent. The phase state $X_k$ holds the library,
graph, policy and flow parameters and verifier posterior counts, and a log
$\mathcal{H}^{\mathrm{imp}}_k$ records each phase's edit decisions and
validation outcomes. Under fixed configuration $\lambda$,
the loop is the recursion $\Xi_{k+1}=\mathfrak{R}_{\lambda}(\Xi_k;\zeta_k)$ on
$\Xi_k=(X_k,\mathcal{H}^{\mathrm{imp}}_k)$, where $\mathfrak{R}_{\lambda}$ runs
training, diagnostics, verification and the library operator $\Phi$, and
$\zeta_k$ is the phase randomness (\Cref{prop:recursion}).

\noindent\textbf{Whether: Verified Evidence and Credible Bounds.}
Each invocation carries a context $z$, and each verified one an independent-verifier
outcome $y_e$ with confidence $c_e$. The verified invocations of skill $u$ in context
$z$, $\mathcal E^{\mathrm{ver}}_{u,z}$, update a hierarchical Beta
posterior over its reliability in $z$:
\begin{equation}
  \alpha^{\mathrm{ver}}_{u,z} = \kappa_u\mu_u + \sum_{e\in\mathcal E^{\mathrm{ver}}_{u,z}}c_e y_e,
  \qquad
  \beta^{\mathrm{ver}}_{u,z}  = \kappa_u(1-\mu_u) +
      \sum_{e\in\mathcal E^{\mathrm{ver}}_{u,z}}c_e(1-y_e),
\end{equation}
where $\mu_u$ is the skill-level reliability and $\kappa_u$ the shrinkage
strength, pooling sparse contexts toward the skill.
Verifier calls are budgeted and go to the invocations with the largest reference edge flow $\varphi^0(e)$ (\Cref{rem:acq}). Decisions then act on
the posterior's lower and upper credible bounds, $\mathrm{LCB}_{u,z}$ and $\mathrm{UCB}_{u,z}$, rather than point estimates.
The interval separates thin evidence from confirmed weakness: a wide interval
from too few verified records triggers \emph{defer} to collect more
observations, whereas $\mathrm{UCB}_{u,z}$ below a quality threshold
triggers \emph{refine} or \emph{prune}.

\noindent\textbf{Which: Flow Share and Signed Utility.}
Both readouts of Section~4.2 enter the library operator together with the
credible bounds and the invocation counts:
\begin{equation}
  \Slib{k+1} = \Phi\bigl(\Slib{k};\
  \{\Psi^0(u), n_u^{\mathrm{call}}, \mathrm{LCB}_{u,z}, \mathrm{UCB}_{u,z}\}_{u\in\Slib{k}},\
  \{\widetilde{A}_e\}_e\bigr),
\end{equation}
where $n_u^{\mathrm{call}}$ is the empirical invocation count of skill $u$. The operator applies
\emph{defer}, \emph{split}, \emph{refine}, \emph{retain} and \emph{prune} to
single skills, \emph{consolidate} to skill pairs, and \emph{generate} to context
clusters, in that order (rules in Appendix~C). Flow share and signed utility rank the candidate edits. A skill is refined when it is sound overall and confirmed weak in some context, split when its contexts diverge, and pruned when it is confirmed weak in every context, has non-positive signed utility, and every context it covers is covered by another skill.
Every edit is versioned and is accepted once paired rollouts are non-inferior on success, tempered reward, token cost, and latency.
$\Slib{k+1}$ changes the graph's edge set (Section~4.1), so the next phase retrains the flow on it and recomputes the readouts and $V_q$.

\begin{table}[t]
\centering\footnotesize
\begin{adjustbox}{max width=\linewidth}\tabcolsep=2.6pt
\renewcommand{\arraystretch}{1.12}
\begin{tabular}{ll|cc|c|c|c|c|cc|c}
\toprule
\textbf{} & \textbf{} & \multicolumn{2}{c}{\textbf{Baseline}} & \textbf{SFT} & \textbf{GRPO$^\dagger$} & \textbf{AFlow} & \textbf{Agent+RL} & \multicolumn{2}{c}{\textbf{Skill evolution}} & \textbf{Ours} \\
\cmidrule(lr){3-4}\cmidrule(lr){5-5}\cmidrule(lr){6-6}\cmidrule(lr){7-7}\cmidrule(lr){8-8}\cmidrule(lr){9-10}\cmidrule(lr){11-11}
\textbf{Dataset} & \textbf{Metric} & \textbf{Qwen3.5} & \textbf{v4-flash} & \textbf{Qwen3.5} & \textbf{Qwen3.5} & \textbf{Qwen3.5} & \textbf{FlowSteer} & \textbf{Trace2Skill} & \textbf{SkillOpt} & \textbf{\RtwoFlow} ($\Delta\uparrow$) \\
\midrule
\multicolumn{11}{@{}l}{\textit{(a) In-Distribution (IID) benchmarks}} \\
\midrule
\multirow{2}{*}{\textbf{HotpotQA}} & Ans EM & 57.97\,\std{0.86} & 71.41\,\std{1.05} & 63.59\,\std{0.43} & 69.53\,\std{0.96} & 86.41\,\std{0.70} & 88.91\,\std{0.35} & 88.12\,\std{0.35} & 89.53\,\std{0.43} & \textbf{91.09}\,\std{0.70}~{\scriptsize\color{evopink}(+33.1)} \\
 & Ans F1 & 73.58\,\std{0.82} & 81.92\,\std{0.43} & 76.86\,\std{0.84} & 81.22\,\std{0.29} & 88.97\,\std{0.18} & 90.30\,\std{0.13} & 90.65\,\std{0.89} & 92.47\,\std{0.10} & \textbf{93.02}\,\std{0.29}~{\scriptsize\color{evopink}(+19.4)} \\
\multirow{2}{*}{\textbf{TriviaQA}} & Ans EM & 44.84\,\std{0.43} & 70.62\,\std{1.18} & 46.25\,\std{0.86} & 47.81\,\std{0.86} & 89.53\,\std{0.70} & 90.31\,\std{0.43} & 89.53\,\std{0.43} & 92.50\,\std{1.18} & \textbf{95.00}\,\std{0.70}~{\scriptsize\color{evopink}(+50.2)} \\
 & Ans F1 & 53.19\,\std{0.14} & 80.75\,\std{0.99} & 59.28\,\std{0.21} & 58.52\,\std{0.21} & 92.50\,\std{0.86} & 93.02\,\std{1.37} & 91.78\,\std{0.73} & 93.71\,\std{1.25} & \textbf{96.97}\,\std{1.01}~{\scriptsize\color{evopink}(+43.8)} \\
\textbf{AIME 2026} & Acc. & 51.33\,\std{1.83} & 48.67\,\std{1.83} & 32.00\,\std{2.98} & 36.00\,\std{2.79} & 53.33\,\std{2.36} & 64.00\,\std{1.49} & 66.67\,\std{0.00} & 67.33\,\std{2.79} & \textbf{76.00}\,\std{1.49}~{\scriptsize\color{evopink}(+24.7)} \\
\textbf{HealthBench} & Rubric & 33.84\,\std{0.68} & 53.43\,\std{0.89} & 31.61\,\std{1.42} & 36.83\,\std{0.39} & 49.32\,\std{0.57} & 52.44\,\std{0.49} & 53.39\,\std{0.06} & 54.53\,\std{0.73} & \textbf{61.89}\,\std{0.44}~{\scriptsize\color{evopink}(+28.0)} \\
\textbf{MBPP+} & Pass@1 & 78.75\,\std{0.35} & 84.69\,\std{0.43} & 77.97\,\std{0.35} & 81.88\,\std{0.86} & 85.00\,\std{1.16} & 87.50\,\std{1.24} & 89.53\,\std{0.70} & 89.22\,\std{0.35} & \textbf{91.25}\,\std{0.86}~{\scriptsize\color{evopink}(+12.5)} \\
\textbf{ALFWorld} & SR & 46.56\,\std{0.70} & 60.47\,\std{0.70} & 39.53\,\std{0.70} & 50.47\,\std{0.70} & 74.06\,\std{0.35} & 80.31\,\std{1.02} & 85.00\,\std{0.35} & 84.53\,\std{0.86} & \textbf{88.75}\,\std{0.43}~{\scriptsize\color{evopink}(+42.2)} \\
\cmidrule(lr){1-11}
\multirow{3}{*}{\textbf{Avg. (IID)}} & Ans EM & 51.41\,\std{0.48} & 71.02\,\std{0.79} & 54.92\,\std{0.48} & 58.67\,\std{0.64} & 87.97\,\std{0.49} & 89.61\,\std{0.28} & 88.83\,\std{0.28} & 91.02\,\std{0.63} & \textbf{93.05}\,\std{0.49}~{\scriptsize\color{evopink}(+41.6)} \\
 & Ans F1 & 63.38\,\std{0.42} & 81.34\,\std{0.54} & 68.07\,\std{0.43} & 69.87\,\std{0.18} & 90.73\,\std{0.44} & 91.66\,\std{0.69} & 91.22\,\std{0.58} & 93.09\,\std{0.63} & \textbf{95.00}\,\std{0.53}~{\scriptsize\color{evopink}(+31.6)} \\
 & Acc. & 52.62\,\std{0.53} & 61.81\,\std{0.55} & 45.28\,\std{0.85} & 51.29\,\std{0.76} & 65.43\,\std{0.68} & 71.06\,\std{0.56} & 73.65\,\std{0.20} & 73.90\,\std{0.76} & \textbf{79.47}\,\std{0.46}~{\scriptsize\color{evopink}(+26.9)} \\
\midrule
\multicolumn{11}{@{}l}{\textit{(b) Out-of-Distribution (OOD) benchmarks}} \\
\midrule
\multirow{2}{*}{\textbf{MuSiQue-Ans}} & Ans EM & 39.06\,\std{0.55} & 49.69\,\std{0.43} & 38.91\,\std{0.65} & 40.62\,\std{0.55} & 77.66\,\std{0.70} & 79.53\,\std{1.28} & 80.94\,\std{0.70} & 81.41\,\std{0.86} & \textbf{85.31}\,\std{0.35}~{\scriptsize\color{evopink}(+46.2)} \\
 & Ans F1 & 47.07\,\std{0.53} & 58.88\,\std{0.89} & 50.59\,\std{1.16} & 50.35\,\std{0.60} & 83.16\,\std{0.72} & 85.63\,\std{0.34} & 85.02\,\std{0.61} & 84.49\,\std{0.63} & \textbf{88.33}\,\std{0.97}~{\scriptsize\color{evopink}(+41.3)} \\
\multirow{2}{*}{\textbf{NQ-Open}} & Ans EM & 23.28\,\std{0.65} & 38.12\,\std{1.16} & 25.94\,\std{0.86} & 23.91\,\std{1.05} & 76.88\,\std{0.89} & 79.84\,\std{0.35} & 81.88\,\std{0.35} & 82.50\,\std{0.43} & \textbf{86.25}\,\std{0.70}~{\scriptsize\color{evopink}(+63.0)} \\
 & Ans F1 & 33.29\,\std{0.81} & 49.91\,\std{0.43} & 37.79\,\std{0.86} & 32.59\,\std{0.98} & 80.62\,\std{0.84} & 84.21\,\std{0.09} & 85.28\,\std{0.35} & 86.87\,\std{1.08} & \textbf{90.16}\,\std{1.25}~{\scriptsize\color{evopink}(+56.9)} \\
\textbf{MATH-Hard} & Acc. & 88.44\,\std{1.16} & 91.88\,\std{0.70} & 88.12\,\std{0.35} & 87.50\,\std{0.00} & 91.41\,\std{0.00} & 93.59\,\std{0.65} & 94.84\,\std{0.70} & 92.19\,\std{0.00} & \textbf{95.94}\,\std{0.86}~{\scriptsize\color{evopink}(+7.5)} \\
\textbf{GPQA} & Acc. & 61.09\,\std{0.35} & 74.06\,\std{1.28} & 68.75\,\std{1.10} & 66.88\,\std{0.70} & 74.06\,\std{1.02} & 83.91\,\std{0.43} & 81.88\,\std{0.65} & 81.88\,\std{0.86} & \textbf{84.84}\,\std{1.18}~{\scriptsize\color{evopink}(+23.8)} \\
\textbf{SWE-bench} & Resolved & 16.41\,\std{0.78} & 34.84\,\std{0.70} & 15.78\,\std{0.35} & 18.91\,\std{0.86} & 25.16\,\std{0.86} & 41.25\,\std{0.86} & 40.62\,\std{1.10} & 42.97\,\std{0.78} & \textbf{44.22}\,\std{1.18}~{\scriptsize\color{evopink}(+27.8)} \\
\textbf{WebShop} & SR & 33.28\,\std{0.43} & 65.31\,\std{1.05} & 32.97\,\std{1.02} & 35.62\,\std{0.43} & 60.16\,\std{0.00} & 79.06\,\std{1.02} & 85.78\,\std{0.86} & 82.50\,\std{0.43} & \textbf{89.53}\,\std{0.89}~{\scriptsize\color{evopink}(+56.2)} \\
\cmidrule(lr){1-11}
\multirow{3}{*}{\textbf{Avg. (OOD)}} & Ans EM & 31.17\,\std{0.43} & 43.91\,\std{0.62} & 32.42\,\std{0.54} & 32.27\,\std{0.59} & 77.27\,\std{0.57} & 79.69\,\std{0.66} & 81.41\,\std{0.39} & 81.95\,\std{0.48} & \textbf{85.78}\,\std{0.39}~{\scriptsize\color{evopink}(+54.6)} \\
 & Ans F1 & 40.18\,\std{0.48} & 54.39\,\std{0.49} & 44.19\,\std{0.72} & 41.47\,\std{0.57} & 81.89\,\std{0.55} & 84.92\,\std{0.18} & 85.15\,\std{0.35} & 85.68\,\std{0.63} & \textbf{89.25}\,\std{0.79}~{\scriptsize\color{evopink}(+49.1)} \\
 & Acc. & 49.80\,\std{0.38} & 66.52\,\std{0.48} & 51.41\,\std{0.39} & 52.23\,\std{0.30} & 62.70\,\std{0.33} & 74.45\,\std{0.39} & 75.78\,\std{0.42} & 74.88\,\std{0.31} & \textbf{78.63}\,\std{0.52}~{\scriptsize\color{evopink}(+28.8)} \\
\bottomrule
\end{tabular}
\end{adjustbox}
\vspace{5pt}
\caption{Main results, five runs, mean\,$\pm$\,standard deviation. All methods share one
frozen Qwen3.5-9B executor; $^\dagger$GRPO trains the Supervisor backbone; $\Delta\uparrow$ is the gain
over Qwen3.5. Averages pool exact match and answer F1 over the
question-answering sets and the primary metric over the rest.}
\label{tab:main}
\end{table}

\noindent\textbf{When: Residual Variance Plateau.}
For a fixed query, $\log Z_\theta(q)$ shifts every full-trajectory residual
$\delta_{0:T}$ by the same constant, so their variance is free of it \citep[log-partition variance;][]{zhang2023robust}:
\begin{equation}
  V_q = \Var_{\tau\mid q}\bigl[\delta_{0:T}(\tau)\bigr].
\end{equation}
$V_q$ is therefore comparable across training steps on one library version: it falls while the policy learns on the current library, and when it levels off the phase ends and the next verifier-gated proposals are made.
Writing $\bar V_w$ for $V_q$ aggregated over held-out queries at step $w$, a phase boundary is declared when the fitted slope of $\bar V_w$ is flat within a tolerance, its relative decrease over a recent window is non-negative but small, and the normalised skill entropy has fallen (\cref{eq:trigger}).

\medskip\noindent\textbf{Proposition 3.} \emph{Every committed skill-library edit is backed by independent verifier evidence and passes paired validation, the readouts of Section~4.2 rank the candidates, and one fixed update maps each phase state to the next.} (Empirically, RQ4 and RQ6 in \S\ref{sec:setup}; formally, Appendix~C.)

\needspace{6\baselineskip}
\section{Experiments}
\label{sec:setup}

We test whether \RtwoFlow outperforms the baselines of \Cref{tab:main} in distribution (\textbf{RQ1}) and out of distribution (\textbf{RQ2}) and transfers across executors (\textbf{RQ3}), what each component (\textbf{RQ4}) and \RtwoTB (\textbf{RQ5}) contribute, and whether it keeps improving with verified edits (\textbf{RQ6}).

\paragraph{Setup.}
Six in-distribution (IID) datasets train the Supervisor and six out-of-distribution (OOD) datasets are held out (\Cref{tab:main}; \citealp{yang2018hotpotqa,joshi2017triviaqa,arora2025healthbench,liu2023your,shridhar2020alfworld,trivedi2022musique,kwiatkowski2019natural,hendrycks2021measuring,rein2023gpqa,jimenez2024swe,yao2022webshop}), with exact match primary for QA.
Baselines are the Supervisor backbone Qwen3.5-9B \citep{team2026qwen3} and DeepSeek-V4-Flash \citep{xu2026deepseek} answering directly, SFT and GRPO \citep{shao2024deepseekmath} on the backbone, AFlow \citep{zhang2025aflow}, FlowSteer \citep{zhang2026flowsteer}, Trace2Skill \citep{ni2026trace2skill} and SkillOpt \citep{yang2026skillopt}; all share the frozen Qwen3.5-9B executor, tools, retrieval corpus and inference budget.

\begin{table}[p]
\centering
\large
\tabcolsep=1pt
\renewcommand{\arraystretch}{1.41}
\setlength{\aboverulesep}{0pt}\setlength{\belowrulesep}{0pt}
\resizebox{\linewidth}{!}{%
\begin{tabular}{@{}l@{\hspace{3pt}}|cccccc|cccccc@{}}
\toprule
\textbf{Variant} & \multicolumn{6}{c|}{\textbf{IID}} & \multicolumn{6}{c}{\textbf{OOD}} \\
\cmidrule(lr){2-7}\cmidrule(lr){8-13}
 & \textbf{HotpotQA} & \textbf{TriviaQA} & \textbf{AIME} & \textbf{HealthB.} & \textbf{MBPP+} & \textbf{ALFWorld} & \textbf{MuSiQue} & \textbf{NQ-Open} & \textbf{MATH} & \textbf{GPQA} & \textbf{SWE} & \textbf{WebShop} \\
 & Ans EM & Ans EM & Acc. & Rubric & Pass@1 & SR & Ans EM & Ans EM & Acc. & Acc. & Resolved & SR \\
\midrule
Qwen3.5-9B (frozen) & 57.97 & 44.84 & 51.33 & 33.84 & 78.75 & 46.56 & 39.06 & 23.28 & 88.44 & 61.09 & 16.41 & 33.28 \\
\RtwoFlow architecture, untrained & 77.03 & 77.81 & 54.00 & 39.53 & 83.75 & 63.44 & 68.28 & 61.09 & 89.84 & 65.94 & 22.97 & 55.31 \\
\midrule
$-$ Shared states (history tree) & 84.06 & 91.09 & 62.00 & 53.75 & 88.12 & 81.09 & 77.19 & 83.12 & 93.28 & 79.06 & 38.12 & 79.38 \\
\midrule
$-$ Learned $\PB$ (uniform $\PB$) & 86.56 & 91.88 & 67.33 & 57.19 & 88.59 & 85.31 & 81.25 & 84.06 & 94.84 & 83.91 & 40.31 & 84.84 \\
$-$ Reference readout (read off the learned flow) & 88.75 & 93.59 & 69.33 & 59.84 & 89.53 & 87.19 & 83.91 & 84.69 & 95.31 & 84.06 & 42.50 & 85.94 \\
$-$ Signed utility (ranked by share alone) & 87.34 & 91.25 & 70.00 & 58.28 & 90.78 & 84.69 & 80.78 & 82.34 & 94.69 & 82.97 & 40.16 & 82.81 \\
\midrule
$-$ Verifier labels (reward-derived labels) & 87.19 & 91.25 & 73.33 & 57.66 & 90.94 & 83.59 & 78.91 & 80.78 & 95.00 & 82.66 & 38.59 & 80.16 \\
$-$ Credible bounds (posterior mean) & 89.06 & 93.91 & 72.00 & 58.91 & 89.84 & 86.88 & 81.09 & 85.31 & 93.59 & 82.81 & 40.47 & 85.00 \\
$-$ Paired non-inferiority check & 88.12 & 93.28 & 69.33 & 59.69 & 89.69 & 86.25 & 83.12 & 83.44 & 95.62 & 83.28 & 41.56 & 84.53 \\
$-$ Plateau trigger (fixed phase length) & 90.16 & 93.12 & 74.67 & 59.84 & 90.00 & 86.72 & 83.44 & 83.44 & 95.31 & 83.44 & 42.34 & 85.94 \\
$-$ Retained phase state (reset each phase) & 86.56 & 89.06 & 68.00 & 58.12 & 89.38 & 82.97 & 78.59 & 78.91 & 94.69 & 82.19 & 39.69 & 81.41 \\
\midrule
\textbf{\RtwoFlow (Full)} & \textbf{91.09} & \textbf{95.00} & \textbf{76.00} & \textbf{61.89} & \textbf{91.25} & \textbf{88.75} & \textbf{85.31} & \textbf{86.25} & \textbf{95.94} & \textbf{84.84} & \textbf{44.22} & \textbf{89.53} \\
\bottomrule
\end{tabular}}
\vspace{5pt}
\caption{Component ablations, five-run mean, protocol of \S\ref{sec:setup}. Each row removes one component of \RtwoFlow, with everything else fixed. The frozen backbone and \RtwoFlow (Full) rows are reproduced from \Cref{tab:main}. The untrained row runs the \RtwoFlow architecture with an untrained Supervisor. $\PB$ is the backward policy, and the flow share is read from the compatible reference flow; the signed utility says whether a skill helps; verifier labels, credible bounds and the paired non-inferiority check gate each library edit; a plateau in the residual variance ends a phase, and the flow and verifier posterior carry into the next. Row blocks follow \S4.1--4.3: graph, flow training and readouts, skill evolution.}
\label{tab:ablation}
\end{table}
\begin{figure}[p]
\centering
\begin{minipage}[b]{0.677\linewidth}\centering
\includegraphics[width=\linewidth,trim=3 7 4 1,clip]{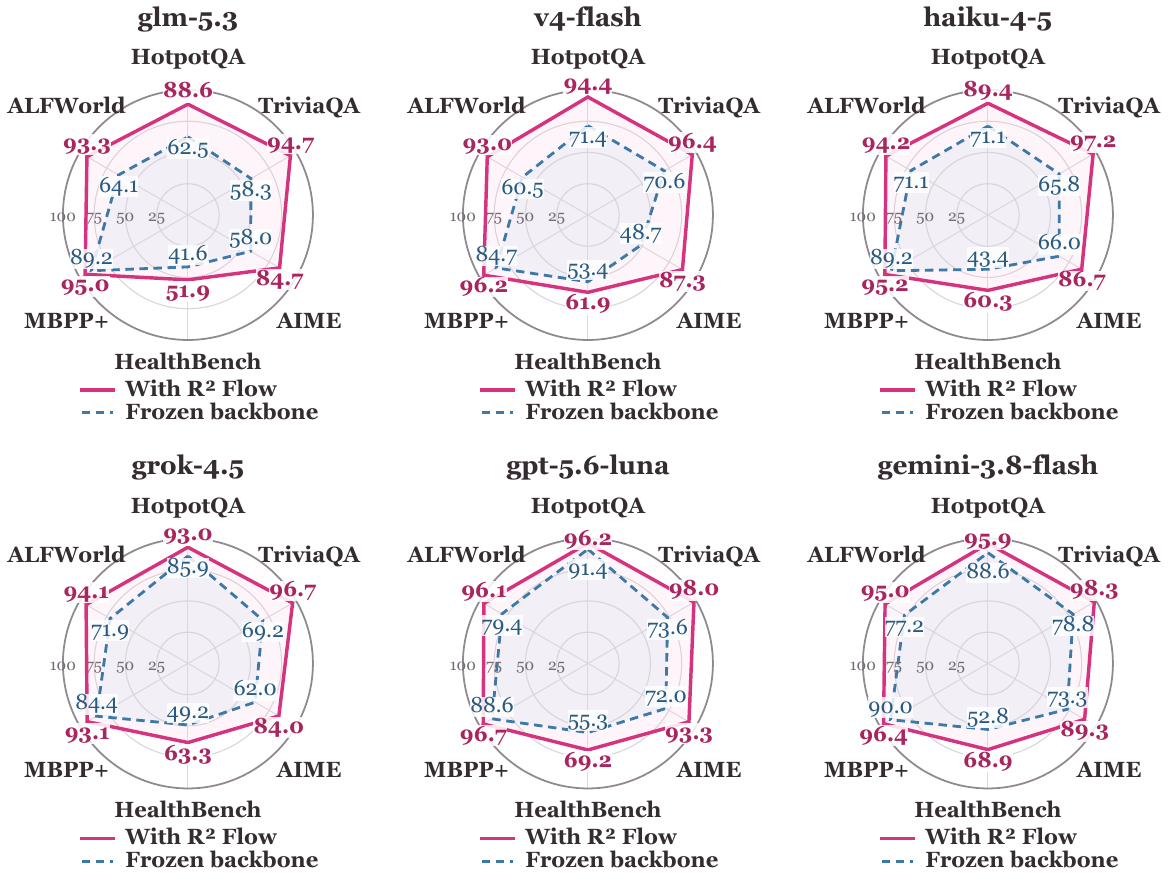}\\[1pt]
{\small (a) IID scores on each frozen executor}
\end{minipage}\hspace{8pt}%
\begin{minipage}[b]{0.300\linewidth}\centering
\includegraphics[width=\linewidth]{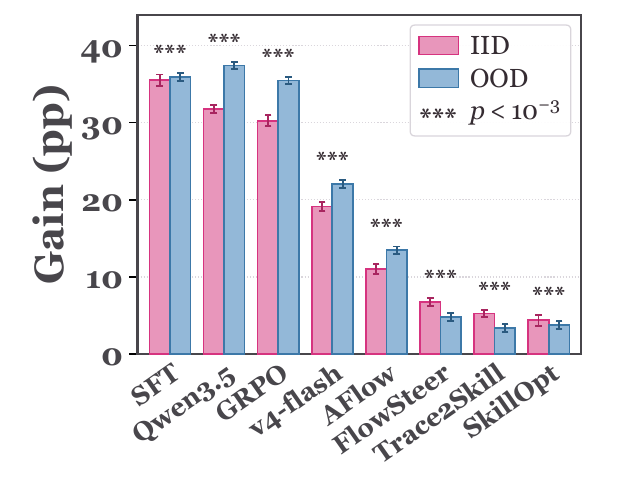}\\[1pt]
{\small (b) Significance of the gains}\\[5pt]
\includegraphics[width=\linewidth]{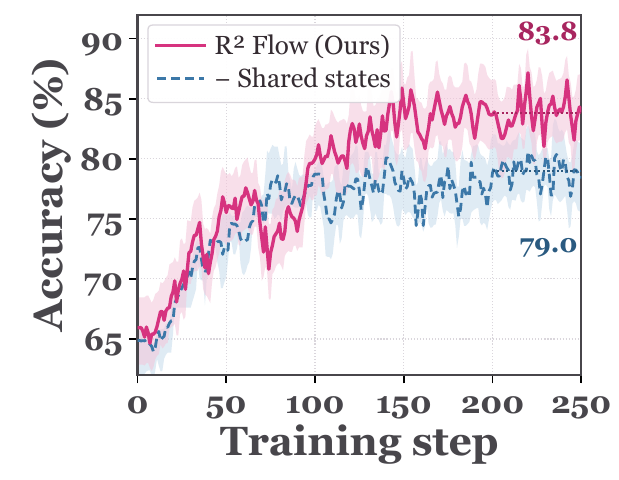}\\[1pt]
{\small (c) Training curve}
\end{minipage}\\[4pt]
\includegraphics[width=1.000\linewidth,trim=1 4 1 4,clip]{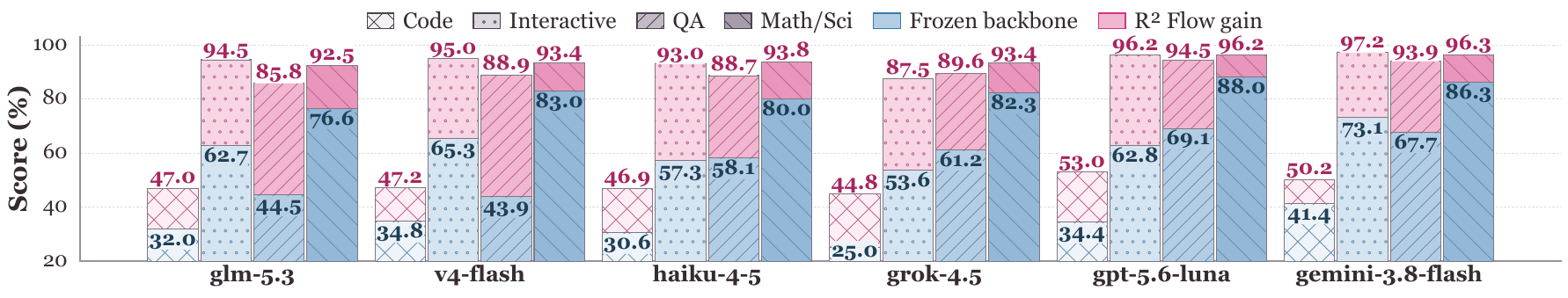}\\[1pt]
{\small (d) OOD scores aggregated by task domain}
\caption{Executor transfer, significance and training dynamics, five-run mean. \textbf{(a)} IID scores of six frozen executors answering directly (dashed) and under \RtwoFlow with the same trained Supervisor (solid); the radius is linear from 0 to 100. \textbf{(b)} Gain of \RtwoFlow over each baseline of \Cref{tab:main}, averaged over the six IID or six OOD datasets on each primary metric, with 95\% confidence intervals; *** marks $p<10^{-3}$ (Welch's $t$-test, five runs). \textbf{(c)} Training accuracy of \RtwoFlow and of $-$~Shared states over 250 steps; shaded: spread across runs; numbers: mean over the last 50 steps. \textbf{(d)} OOD scores averaged per domain: QA (MuSiQue, NQ-Open), Math/Sci (MATH, GPQA), Code (SWE-bench), Interactive (WebShop); each bar stacks the frozen score and the \RtwoFlow gain.}
\label{fig:transfer}
\end{figure}
\begin{figure}[t]
\centering
\sbox{\objtabbox}{\setlength{\tabcolsep}{2.4pt}\renewcommand{\arraystretch}{1.25}%
\resizebox{0.70\linewidth}{!}{%
\begin{tabular}{@{}l|cccccc|@{\hspace{4pt}}cc@{\hspace{4pt}}|@{\hspace{4pt}}cc@{}}
\toprule
\textbf{Objective} & \multicolumn{6}{c|}{\textbf{IID score}} & \multicolumn{2}{c|}{\textbf{Learner, per step}} & \multicolumn{2}{c}{\textbf{Whole step}} \\
\cmidrule(lr){2-7}\cmidrule(lr){8-9}\cmidrule(l){10-11}
 & \textbf{HotpotQA} & \textbf{TriviaQA} & \textbf{AIME} & \textbf{HealthB.} & \textbf{MBPP+} & \textbf{ALFWorld} & \textbf{Tokens} & \textbf{GPU} & \textbf{GPU} & \textbf{Wall-clock} \\
 & Ans EM & Ans EM & Acc. & Rubric & Pass@1 & SR & & s & s & s \\
\midrule
GRPO & 84.22 & 92.66 & 64.67 & 53.93 & 86.56 & 79.06 & 2,510,218 & 252 & 2,406 & 788 \\
\midrule
DB & 87.03 & 93.12 & 65.33 & 56.93 & 88.28 & 82.50 & 3,121,649 & 441 & 2,595 & 834 \\
TB & 88.59 & 93.91 & 70.67 & 57.52 & 89.69 & 85.16 & 3,296,311 & 389 & 2,542 & 822 \\
TTB & 88.12 & 94.06 & 70.00 & 58.47 & 90.16 & 85.31 & 3,296,346 & 390 & 2,544 & 823 \\
SubTB & 89.38 & 93.91 & 69.33 & 59.24 & 90.94 & 87.66 & 3,121,927 & 444 & 2,598 & 837 \\
\midrule
\textbf{\RtwoTB (Ours)} & \textbf{91.09} & \textbf{95.00} & \textbf{76.00} & \textbf{61.89} & \textbf{91.25} & \textbf{88.75} & 3,121,184 & 433 & 2,587 & 832 \\
\bottomrule
\end{tabular}}}%
\begin{minipage}[b]{0.70\linewidth}\centering
\raisebox{\dp\objtabbox}{\usebox{\objtabbox}}\\[\dimexpr 0.0905\ht\objtabbox+0.0905\dp\objtabbox+3pt\relax]
{\small (a) IID scores and training cost per step}
\end{minipage}\hfill
\begin{minipage}[b]{0.29\linewidth}\centering
\includegraphics[height=\dimexpr 1.0905\ht\objtabbox+1.0905\dp\objtabbox\relax]{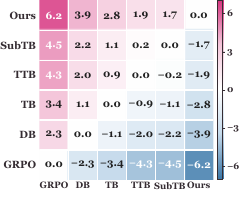}\\[3pt]
{\small (b) Pairwise OOD differences}
\end{minipage}\\[6pt]
\begin{minipage}[t]{0.242\linewidth}\centering
\includegraphics[width=\linewidth]{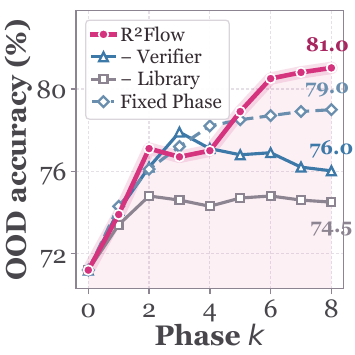}\\[-1pt]
{\small (c) Phase-to-phase gain}
\end{minipage}\hfill
\begin{minipage}[t]{0.242\linewidth}\centering
\includegraphics[width=\linewidth]{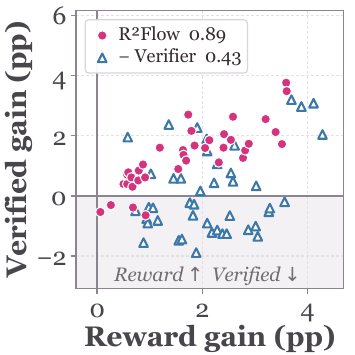}\\[-1pt]
{\small (d) Reward vs.\ verified}
\end{minipage}\hfill
\begin{minipage}[t]{0.242\linewidth}\centering
\includegraphics[width=\linewidth]{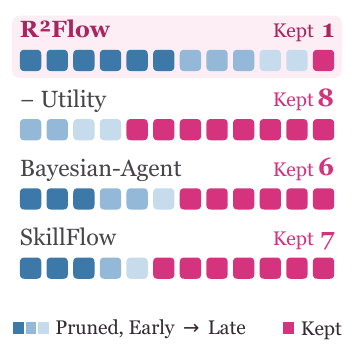}\\[-1pt]
{\small (e) Fate of harmful skills}
\end{minipage}\hfill
\begin{minipage}[t]{0.242\linewidth}\centering
\includegraphics[width=\linewidth]{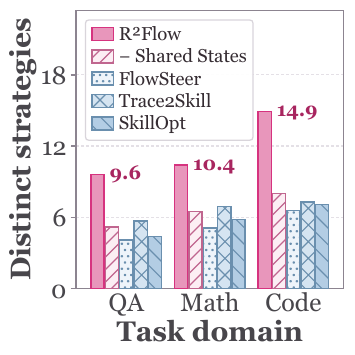}\\[-1pt]
{\small (f) Distinct strategies}
\end{minipage}
\caption{Training objectives and recursive self-improvement, five-run mean. \textbf{(a)} IID scores and per-step cost of each Supervisor objective (TB, \citealp{malkin2022trajectory}; DB, \citealp{bengio2023gflownet}; SubTB without $b_d$, \citealp{madan2023learning}; TTB, \citealp{zhang2026skillflow}). \textbf{(b)} Row minus column objective in mean OOD score. \textbf{(c)} Mean OOD score per phase; $-$~Verifier uses reward-derived labels, $-$~Library freezes the library, and Fixed Phase replaces the plateau trigger. \textbf{(d)} Training-reward gain against verified gain per committed edit; shaded: held-out verified score falls. \textbf{(e)} Phase at which each of twelve injected harmful skills is pruned. \textbf{(f)} Distinct strategies per domain, reorderings counted once.}
\label{fig:rsi-panels}\label{tab:objectives}
\end{figure}

\paragraph{In distribution (RQ1).}
\RtwoFlow leads every IID dataset and metric (\Cref{tab:main}), 2.03 EM and 5.57 accuracy points above SkillOpt, the strongest baseline, with every average gain significant at $p<10^{-3}$ (\Cref{fig:transfer}b).
SFT and GRPO$^\dagger$ on the backbone fall below Qwen3.5 in accuracy, whereas the \RtwoFlow architecture adds 13.71 IID points and training 18.07 more (\Cref{tab:ablation}). Tuning the backbone on task reward also trades one task for another: GRPO$^\dagger$ raises the IID exact-match average by 7.26 points but lowers AIME 2026 from 51.33 to 36.00, while \RtwoFlow raises both.

\paragraph{Out of distribution (RQ2).}
\RtwoFlow again leads every OOD row, 3.83 EM and 2.85 accuracy points above the strongest baselines, with larger gains over Qwen3.5 than in distribution and its narrowest margins, about one point, on GPQA and MATH-Hard.
Verification matters more here: edits admitted on reward-derived labels cost 5.00 OOD points against 3.34 IID (\Cref{tab:ablation}), as most of them raise the training reward while lowering the held-out verified score (RQ6). With a 9B Supervisor and executor, \RtwoFlow also beats DeepSeek-V4-Flash answering directly on all twelve datasets.

\paragraph{Across executors (RQ3).}
With the trained Supervisor fixed and six other frozen executors (\Cref{fig:transfer}a,d), \RtwoFlow raises every backbone on every IID dataset and OOD domain, by 15.70 to 24.64 points on average, and narrows the IID spread across executors from 14.5 to 6.9 points.
Weaker executors gain more (Pearson $r=-0.97$) and all gain more out of distribution, consistent with the orchestration and verified skills compensating for the weaker executors.

\paragraph{Component ablations (RQ4).}
All 108 ablated cells of \Cref{tab:ablation} fall below the full model.
Shared states matter most, 7.31 IID and 5.99 OOD points: on a history tree, reorderings of one procedure no longer pool evidence, so the flow spreads over them and distinct strategies per domain fall from 9.6--14.9 to 5.2--8.0, near the baselines (\Cref{fig:rsi-panels}f), and training accuracy falls behind (\Cref{fig:transfer}c).
For credit, a uniform $\PB$, which splits credit evenly over the skills reaching a merged state, costs 4.52 IID points, and ranking by share alone, which reads frequent use as benefit, 3.61.
Replacing credible bounds by posterior means, which cannot tell thin evidence from confirmed weakness, costs 2.23 IID and 2.97 OOD points, and resetting the phase state, which discards the flow and verifier posterior carried across phases, about five points on both splits. Reading the shares off the learned flow, whose internal flows depend on $\PB$, costs 2.63 IID and 1.61 OOD points.

\paragraph{Training objectives (RQ5).}
\RtwoTB leads every IID dataset (\Cref{fig:rsi-panels}a) and beats every other objective on OOD by 1.69--6.17 points (\Cref{fig:rsi-panels}b), for 7.5\% more GPU time than GRPO.
SubTB, which balances the same sub-trajectories without the per-domain bias $b_d$ that absorbs a shared offset of each domain's terminal residuals, comes closest. TB and DB, which balance only whole trajectories or single edges, trail it, and GRPO, without a balance residual, is lowest.

\paragraph{Phase-to-phase improvement (RQ6).}
The mean OOD score rises from 71.2 to 81.02 over eight phases and dips only once (\Cref{fig:rsi-panels}c), as each phase trains on the library the previous one committed and 33 of the 37 committed edits raise the held-out verified score (precision 0.89, \Cref{fig:rsi-panels}d).
With reward-derived labels, 25 of 44 edits raise the training reward but lower the verified score, so accuracy peaks at phase 3 and then declines; a frozen library stalls below 75, and fixed-length phases end at 78.99.
Signed utility lets \RtwoFlow prune eleven of twelve injected harmful skills (\Cref{fig:rsi-panels}e), six by phase 2 and the rest as evidence accumulates, whereas share-only ranking keeps eight, Bayesian-Agent \citep{wu2026bayesian} six and SkillFlow \citep{zhang2026skillflow} seven.

\section{Conclusion}

We presented \RtwoFlow, a recursive self-improvement framework that admits library edits on verifier evidence outside the task reward. On twelve datasets it beats every baseline in and out of distribution and improves over eight phases without human intervention, at 0.89 held-out edit precision.

\clearpage
\section*{Ethics Statement}
This work follows the ICLR Code of Ethics. We report our results
without fabrication, falsification or intentional misrepresentation, and
Section~\ref{sec:setup} and Appendix~\ref{app:full-results} give the
implementation and evaluation details needed to verify them. All twelve
datasets are public benchmarks, used for their intended research purposes and
under their licenses, and we cite the work and resources we build on.
The study involves no human participants and collects no personal information.
HealthBench serves only as a benchmark, and \RtwoFlow is not intended to give
medical advice. Because \RtwoFlow revises its own skill library, it commits an
edit only on independent verifier evidence and after paired validation.

\section*{Reproducibility Statement}
Section~\ref{sec:setup} and Appendix~\ref{app:full-results} specify the twelve
datasets with their fixed evaluation draws and seed, the metrics, the baselines,
the frozen Qwen3.5-9B executor and the validation query set used for phase
control. Appendix~\ref{app:flow} gives the \RtwoTB objective with its weighting
and tempering constants, Appendix~\ref{app:evolve} gives the verifier-gated
evolution rules, the phase trigger and the paired validation, and
Appendices~\ref{app:graph}--\ref{app:proof-details} state the assumptions and
give the proofs behind Propositions~1--3. Every reported number is a mean over
five runs. The complete training code, the model weights and the datasets are available
at \url{https://github.com/beita6969/r2flow}.

\section*{AI Use Statement}
In this work, we used generative AI tools to improve the readability of the text,
to assist the literature survey and format the references, to refine the formal
statements and proofs in Appendices~\ref{app:graph}--\ref{app:proof-details}, to
draft the text that interprets our experimental results, to give feedback on the
design of the method and the experiments, and to prepare figures. We have not
used generative AI tools to generate data or to produce the reported results. We
take responsibility for the final content of this work, including text, claims
or artifacts produced with the aid of generative AI.

\bibliography{refs}

\clearpage
\appendix
\numberwithin{proposition}{section}
\numberwithin{definition}{section}
\numberwithin{assumption}{section}
\numberwithin{remark}{section}
\numberwithin{equation}{section}
\setcounter{proposition}{0}
\setcounter{definition}{0}
\setcounter{assumption}{0}
\setcounter{remark}{0}
\section{Formal Foundations of the Orchestration Graph}
\label{app:graph}
\subsection{Quotient Orchestration Graph}
\label{sec:graph}

\paragraph{Raw histories.}
An execution history evolves by appending atomic records,
\begin{equation}
  H_t \;=\; H_{t-1} \oplus (r_t, e_t, o_t^{\ex}),
  \label{eq:append}
\end{equation}
where $r_t$ is internal reasoning text, $e_t$ the executed orchestration event
(a skill or tool invocation with structured arguments), and $o_t^{\ex}$ the
environment feedback. Given a query $q$ and a retrieved skill subset
$\mathcal{S}_{\mathrm{ret}}\subseteq\Slib{k}$ together with a task-conditioned
prompt state $\omega_q$, the root history is
$H_0 = q \oplus \mathcal{S}_{\mathrm{ret}} \oplus \omega_q$, to which \cref{eq:append} appends the first atomic record.

\begin{remark}[Canonicalisation creates informative backward structure]
\label{rem:tree}
If nodes are identified with histories, every node except the root has exactly
one parent, namely $H_t$ with its last record deleted. The resulting object is a
prefix tree, not a confluent DAG, and the backward transition probability is
degenerate, $\PB(H_{t-1}\mid H_t)\equiv 1$. Canonicalisation merges reorderings
of independent events into shared states, turning the prefix tree of raw
histories into a DAG with multiple valid incoming edges; the resulting backward
transition therefore carries information about how a state was reached.
\end{remark}
\chg{L37: made explicit. The original text called the history graph a DAG; it is
a trie. Stating this openly is stronger than leaving it for a reviewer to find.}

\begin{definition}[Execution trace and dependency order]
\label{def:trace}
Let $E_t$ be the set of atomic event instances committed after $t$ steps.
The direct dependency relation records when an event consumes an artifact, fact,
or environment change produced by another event; $\prec_t$ denotes its
transitive closure, so it is a partial order. Repeated calls with the same skill and arguments are
distinct event instances, while their provenance is retained in the canonical
state. Let $\mathcal{A}_t$ be
the canonicalised artifacts with provenance, $\mathcal{V}_t$ the verified facts,
and $\Omega_t$ the environment observables required for future decisions. The
abstraction map, whose values are the shared states of Section~4.1, is
\begin{equation}
  s_t \;=\; \sigma(H_t) \;=\; \bigl(E_t,\ \prec_t,\ \mathcal{A}_t,\
  \mathcal{V}_t,\ \Omega_t\bigr).
  \label{eq:sigma}
\end{equation}
Two histories are identified iff they are distinct linear extensions of the same
dependency order with identical artifacts, facts, and observables; this is the
Mazurkiewicz trace equivalence of concurrency theory \citep{mazurkiewicz1986trace}
(trace in the sense of a partial-order equivalence class of executions, not an
orchestration trace in the sense used by \citet{zhang2026reinforcement}).
\end{definition}
\chg{New. Replaces the implicit ``node $=$ full history''. This
dependency-aware trace quotient is the object the contribution should be
claimed on.}

\begin{assumption}[Quotient validity]
\label{ass:quotient}
For all $H,H'$ with $\sigma(H)=\sigma(H')=s$:
\begin{enumerate}[label=(\alph*),leftmargin=*]
  \item \emph{action-mask consistency}: $\mathcal{E}(H)=\mathcal{E}(H')$, the
        legal event sets coincide;
  \item \emph{successor consistency}:
        $\Prob\bigl(\sigma(H\oplus e)=s'\mid H,e\bigr)
        =\Prob\bigl(\sigma(H'\oplus e)=s'\mid H',e\bigr)$ for all legal $e$;
  \item \emph{terminal-reward consistency}: if $s$ is terminal then
        $R(H)=R(H')$, i.e.\ the reward is measurable with respect to the
        terminal \emph{node} $s$ shared by $H$ and $H'$.
\end{enumerate}
\end{assumption}
\chg{New, and mandatory. Without (a)--(c) the quotient is not a Markov
abstraction and no flow identity below is well posed. Violations are quantified
in \cref{sec:scope}, not assumed away.}

\begin{assumption}[Markov policy and flow]
\label{ass:markov}
The backward policy and state flow read only the abstract state, and the forward
policy reads it together with the current reasoning text $r_t$: the first two are
functions of $(s,q,k)$ and the forward policy of $(s,q,k,r_t)$, none of the
discarded history (\cref{sec:scope}). Reasoning text
$r_t$ influences future behaviour only through the canonicalised content it
commits to the artifacts $\mathcal{A}_t$ or the verified facts $\mathcal{V}_t$.
\end{assumption}
\chg{New. If the LLM keeps conditioning on the raw transcript, merged nodes emit
different action distributions and the quotient DAG exists only on paper.}

\begin{assumption}[Finite event horizon and branching]
\label{ass:finite}
There is a finite horizon $T<\infty$, and every nonterminal abstract state has a
finite nonempty set of legal events. Terminal states have no outgoing events.
The horizon and event set are fixed within a library version; a phase update may
change them only at a phase boundary declared by the trigger of \cref{eq:trigger}.
\end{assumption}
\chg{The finite-horizon and finite-branching conditions are used explicitly in
the DAG, backward-induction, and finite-sum arguments below. They are not a claim
about unbounded tool use.}

\begin{assumption}[Canonicalised transitions]
\label{ass:det}
After canonicalisation, the successor state is a deterministic function of
$(s,e)$. Tools with stochastic canonicalised output are modelled by the event/environment split $s\to(s,e)\to s'$ into an event step and an environment step.
\end{assumption}
\chg{New. The original $\PF$ omitted both the probability of $r_t$ and the
environment transition kernel; this assumption states the scope in which that
omission is legitimate.}

\begin{definition}[Edges, in-edges, rank]
\label{def:edges}
An edge $s\xrightarrow{e}s'$ exists iff $e$ is legal at $s$ and
$\sigma$-appending $e$ yields $s'$. The in-edge set is
$\Inn(s')=\{(s,e): s\xrightarrow{e}s'\}$, and its predecessor-state
projection is:
\begin{equation}
  \Pa(s') \;=\; \bigl\{\,s:\exists e\text{ with }s\xrightarrow{e}s'\,\bigr\}.
  \label{eq:parents}
\end{equation}
Each edge commits one atomic event, so each in-edge deletes one $\prec$-maximal
event whose removal leaves a legal predecessor.
The rank is $\rho(s)=|E_s|$, the number of committed atomic events.
\end{definition}

\begin{proposition}[Well-posedness]
\label{prop:dag}
Every edge strictly increases $\rho$; hence the quotient graph is a finite DAG
with a single source $s_0=\sigma(H_0)$, and every complete trajectory visits
states of strictly increasing rank $\rho$, the number of committed atomic events.
\end{proposition}
\begin{proof}
Each edge commits at least one atomic event, so $\rho(s')\ge\rho(s)+1>\rho(s)$.
A strictly monotone integer potential forbids cycles. By
\cref{ass:finite}, every path has at most $T$ events and every node has finite
branching, so the reachable quotient graph is finite.
\end{proof}
\chg{Replaces the earlier device of putting the step index $t$ into the node
identity. With macro-events $t\neq|E_t|$, and $(t,C)$ would split one outcome
into several terminal nodes --- see \cref{prop:multiplicity}.}

\paragraph{Terminal states and tempered reward.}
Let $\mathcal{X}$ be the set of terminal nodes. By
\cref{ass:quotient}(c) the reward is a function of the terminal node, and we
temper and shift it as
\begin{equation}
  \Rt(x) \;=\; \bigl(R(x)+\eps\bigr)^{\eta},
  \qquad R(x)\in[0,1],\ \eps>0,\ \eta>0.
  \label{eq:reward}
\end{equation}
\chg{L49--52: reward was written $R(\tau)$ and the temperature was $\beta$.
After merging, several trajectories share a terminal node, so a trajectory-level
reward makes the terminal anchor ill-posed. Also $\beta\to\eta$ to free
$\beta$ for the Beta posterior in \cref{sec:bayes}.}

\begin{proposition}[Representation multiplicity]
\label{prop:multiplicity}
Suppose an observable outcome $o$ is represented by $k(o)$ distinct terminal
nodes, each anchored to $\Rt(o)$. Then at any exact balance solution
\begin{equation}
  P_F^{\mathrm{term}}(o\mid q) \;=\; \frac{k(o)\,\Rt(o)}{Z^\star(q)} ,
  \label{eq:multiplicity}
\end{equation}
i.e.\ outcomes are over-sampled in proportion to how many ways they can be
represented.
\end{proposition}
\begin{proof}
At the exact solution $P_F^{\mathrm{term}}(x\mid q)=\Rt(x)/Z^\star(q)$ for each terminal node $x$
(\cref{prop:pb}). Summing over the $k(o)$ nodes representing $o$ gives
\cref{eq:multiplicity}.
\end{proof}

\paragraph{Library versioning.}
A skill-library update changes the legal event set and therefore the graph
itself. We write $\Gph{k}$ for the graph of library version $k$ and condition all
learned objects on the version index, the policies $\PF(\cdot\mid s,q,k)$ and
$\PB(\cdot\mid s',q,k)$, and the flow $\Fpsi(s,q,k)$.
\chg{L146--158: the original treated $\Phi$ as an in-place edit of a fixed DAG.
Each evolution step produces a new graph.}

\subsection{Forward and Backward Policies}
\label{sec:policies}

\paragraph{Semantic action factorisation.}
An event is a pair $e=(u,\mathrm{args})$ of a skill/tool identifier
$u\in\Slib{k}$ and structured arguments, so that
\begin{equation}
  \PF(e\mid s,q,k) \;=\; \PF(u\mid s,q,k)\,\PF(\mathrm{args}\mid u,s,q,k).
  \label{eq:pf}
\end{equation}
\chg{L61--63: $\PF$ was $\pi_\theta(a_t\mid r_t,H_{t-1})$, a free-text
distribution conditioned on the raw history. \cref{eq:pf} is a proper
categorical distribution over an enumerable event space and is
$\sigma$-measurable, as \cref{ass:markov} requires.}

\begin{remark}[Token-length normalisation]
\label{rem:lengthnorm}
Because each event is modelled directly as a skill identifier and structured
arguments, \cref{eq:pf} defines a normalised probability over executable events
without token-length correction, and a policy used consistently for sampling
and for balance evaluation satisfies the balance identities encoded by the residual in \cref{eq:residual}.
\end{remark}
\chg{L68--73: equation (6) deleted from the main development and demoted to this
remark.}

\paragraph{Backward policy.}
The backward policy is a distribution over \emph{in-edges}, parameterised with
hindsight features and normalised explicitly:
\begin{equation}
  \PB\bigl((s,e)\mid s'\bigr)
  \;=\;
  \frac{\exp f_\phi(s,e,s')}
       {\displaystyle\sum_{(\bar s,\bar e)\in\Inn(s')}\exp f_\phi(\bar s,\bar e,s')}.
  \label{eq:pb}
\end{equation}
\chg{L65--67: the original $P_\phi(a_t\mid H_{t-1}\oplus o_t^{\ex})$ is a
distribution over \emph{actions}, not over parents, and after token averaging it
is not normalised at all. Normalisation over in-edges (rather than parent
states) keeps two different skills reaching the same child distinguishable,
which skill-level credit requires.}

\begin{definition}[Canonical backward policy]
\label{def:pb0}
$\PBz\bigl((s,e)\mid s'\bigr)=1/|\Inn(s')|$: the valid in-edges are
weighted uniformly \emph{locally}. This uniform weighting is local, so complete linearisations of a trace can have different probabilities under the canonical backward policy $\PBz$.
\end{definition}

\begin{proposition}[Two faces of $\PB$]
\label{prop:pb}
Let $Z^\star(q)=\sum_{x\in\mathcal{X}}\Rt(x)$ and suppose balance holds exactly.
\begin{enumerate}[label=(\roman*),leftmargin=*]
  \item \emph{(Invariance.)} The terminal marginal satisfies
        $P_F^{\mathrm{term}}(x\mid q)=\Rt(x)/Z^\star(q)$ for every choice of normalised backward policy $\PB$ with full
        support.
  \item \emph{(Dependence.)} The internal state flows $F(s)$ do depend on $\PB$.
\end{enumerate}
\end{proposition}
\begin{proof}
(i) Exact balance gives $\PF(\tau)=\Rt(x_\tau)\PB(\tau\mid x_\tau)/Z^\star(q)$.
Summing over trajectories reaching a fixed $x$ and using
$\sum_{\tau\to x}\PB(\tau\mid x)=1$ yields
$P_F^{\mathrm{term}}(x\mid q)=\Rt(x)/Z^\star(q)$.
(ii) Take a single terminal $x$ reached by two length-two paths through
distinct states $\hat u$ and $\hat v$. Then $F(\hat u)/F(\hat v)
=\PB(\hat u\text{-path}\mid x)/\PB(\hat v\text{-path}\mid x)$, which equals $9$
for $\PB=(0.9,0.1)$ and $1$ for $\PBz$, while the terminal marginal is
identical in both cases.
\end{proof}

\begin{remark}[Consequence for credit]
\label{rem:credit}
Terminal reward matching does not pin down internal flow: edge and state flows
read off $F$ depend on the backward policy, whereas the skill share $\Psi^0$ does
not (\cref{cor:share-estimator}). We keep training and readout separate: the flow
is trained with the learned $P_{B,\phi}$, and the reference readout is defined by
a compatible triple $(F^0,\PF^0,\PBz)$ and estimated from the trained policy's
rollouts with the residual weights of \cref{cor:share-estimator}.
\end{remark}
\chg{Central correction. An earlier draft argued that normalising $\PB$ makes
credit identifiable; it does not. A reference read-out is well-defined only
when the forward flow is reconstructed compatibly with the chosen $\PBz$.}

\section{Balance Objectives and Flow Diagnostics}
\label{app:flow}
\subsection{The \RtwoTB Objective}
\label{sec:flow-objective}

We learn a forward policy $\PF$, a backward policy $\PB$, and an explicit
log-flow head $\log\Fpsi(s)$. For $0\le i<j\le T$ define the residual
\begin{equation}
  \dbres_{i:j}
  \;=\;
  \log\Fpsi(s_i)
  +\sum_{t=i+1}^{j}\log\PF(e_t\mid s_{t-1})
  -\log\Fpsi(s_j)
  -\sum_{t=i+1}^{j}\log\PB\bigl((s_{t-1},e_t)\mid s_t\bigr),
  \label{eq:residual}
\end{equation}
with the terminal anchor and source identification
\begin{equation}
  \log\Fpsi(x)\;=\;\log\Rt(x)\quad\forall x\in\mathcal{X},
  \qquad
  \Fpsi(s_0)\;=\;Z_\theta(q).
  \label{eq:anchor}
\end{equation}
The objective is
\begin{equation}
  \LTB
  \;=\;
  \sum_{0\le i<j\le T}\omega_{j-i}\,\dbres_{i:j}^{2},
  \label{eq:flow-objective}
\end{equation}
where $j=i+1$ recovers detailed balance, $(i,j)=(0,T)$ recovers trajectory
balance, and intermediate lengths interpolate between local bootstrapping and
whole-trajectory variance. \RtwoTB is this objective with the domain offset of
\cref{rem:domainbias} and the $r_t$-conditioned $\PF$ of \cref{sec:scope}.
\chg{L76--83: replaces $\Delta(\tau)$ and $\mathcal{L}_{TTB}=(\Delta/T)^2$.
Length weighting may be kept inside $\omega$, but see \cref{rem:rawdelta}: the
phase diagnostic must use the \emph{undivided} residual.}

\begin{remark}[Domain bias]
\label{rem:domainbias}
For a query of domain $d$ the implementation adds one scalar $b_d$ to
$\log Z_\theta(q)$ and to every intermediate $\log\Fpsi(s)$, and keeps the
terminal anchor at $\log\Rt(x)$. The offset cancels in every residual whose
endpoints are both nonterminal and shifts every residual that ends at a terminal
by $b_d$. Since the flow head is conditioned on $q$, $b_d$ keeps the zero-loss solutions
of \cref{prop:subtb-exact} and absorbs the shared domain offset during training.
It is updated toward a
closed-form target: after each step it moves by $0.3$ times the median, over the
domain's trajectories, of the shift
$c^\star=-\sum_{i}\omega_{T-i}\dbres_{i:T}/\sum_{i}\omega_{T-i}$ that
minimises each trajectory's loss, and it is carried to the next phase together
with $\psi$. At zero loss every $c^\star$ vanishes, so $b_d$ is stationary. The
within-query variance $V_q$ and the reference share $\Psi^0(u)$ do not depend on
$b_d$. The implementation uses
$\omega_{j-i}\propto 0.9^{\,j-i}$ normalised per trajectory, $\eta=4$, and
$\epsilon=0.1$.
\end{remark}

\begin{remark}[Which residual to report]
\label{rem:rawdelta}
$\log Z_\theta(q)$ enters $\dbres_{0:T}$ as a constant shift for a fixed $q$,
but $\log Z_\theta(q)/T$ does not when trajectory lengths vary. All diagnostics
in \cref{sec:phase} therefore use the raw residual $\dbres_{0:T}$ of \cref{eq:residual}, never a
length-normalised variant.
\end{remark}

\subsection{Flow Diagnostics}
\label{sec:diag}

\begin{proposition}[What the state flow is]
\label{prop:flow}
Let $h_B(s\mid x)$ be the probability that backward sampling from terminal $x$
under $\PB$ passes through $s$. At exact balance,
\begin{equation}
  F(s)\;=\;\sum_{x\in\mathcal{X}}\Rt(x)\,h_B(s\mid x),
  \qquad F(s_0)=Z^\star(q),\qquad F(x)=\Rt(x).
  \label{eq:flowid}
\end{equation}
\end{proposition}
\begin{proof}
Trajectory flow at the optimum is $F(\tau)=Z^\star\PF(\tau)=\Rt(x_\tau)\PB(\tau\mid
x_\tau)$. Summing over trajectories through $s$ and grouping by terminal state
gives \cref{eq:flowid}.
\end{proof}

\begin{remark}[The state flow as downstream reward mass]
\label{rem:notvalue}
The flow through $s$ that terminates at $x$ equals both $F(s)\PF(x\mid s)$ and
$\Rt(x)h_B(s\mid x)$, so $\PF(x\mid s)=\Rt(x)h_B(s\mid x)/F(s)$ and
\begin{equation}
  \E_{\PF}\bigl[\Rt\mid s\bigr]
  \;=\;\frac{1}{F(s)}\sum_{x}\Rt(x)^{2}\,h_B(s\mid x).
  \label{eq:notvalue}
\end{equation}
The second moment appears, so $F(s)\not\propto\E[\Rt\mid s]$ in general. $F$ is a
partition function over the downstream reward mass of $s$ and $\log F$ a free energy.
\end{remark}
\chg{L96--101: equation (10), $F(H_t)=\E[\exp(\sum\log I)\mid H_t]$, is deleted.
With full histories as nodes the conditioning fixed every term inside the
expectation, so the operator was vacuous. \cref{eq:flowid} is the correct
object. The sentence about a ``distributional estimator rather than a conserved
scalar field'' is also deleted: after the repair the flow \emph{is} conserved,
and the disclaimer worked against the method.}

\paragraph{Importance: normalised flow share.}
Define the normalized expected invocation mass routed through an edge and
through a skill under the reward-matching reference law,
\begin{equation}
  \varphi^0(e)\;=\;\frac{F^0(s)\,\PF^0(e\mid s)}{Z^\star(q)},
  \qquad
  \Psi^0(u)\;=\;\frac{\sum_{e\,\text{invoking}\,u}\varphi^0(e)}
                   {\sum_{e\in\mathcal E_{\mathrm{skill}}}\varphi^0(e)}.
  \label{eq:share}
\end{equation}
\chg{L103--107: replaces $\hat F(s)=\E_\tau[\sum_t F(H_t)]$, which summed a
strictly positive quantity and therefore confounded ``used often'' with
``matters per use'', and could never express harm.}

\begin{proposition}[Conservation]
\label{prop:cut}
If $\mathcal{C}$ is a cut, i.e.\ every complete trajectory visits exactly one
state of $\mathcal{C}$, then $\sum_{s\in\mathcal{C}}F^0(s)=Z^\star(q)$. For every nonterminal
state $s$, $\sum_{e\in\Out(s)}\varphi^0(e)=F^0(s)/Z^\star(q)$. If
$\mathcal{E}_{\mathrm{skill}}$ is the set of counted edges and every such edge
invokes exactly one skill, then
$\sum_{u}\Psi^0(u)=1$ when the denominator of $\Psi^0$ is restricted to
$\mathcal{E}_{\mathrm{skill}}$.
\end{proposition}
\begin{proof}
For each terminal-directed trajectory $\tau$, let its reference trajectory mass be
$m^0(\tau)=F^0(s_0)\PF^0(\tau)$. Exact balance and the terminal anchor give
$\sum_{\tau}m^0(\tau)=Z^\star(q)$. Since the cut is visited exactly once,
partitioning the trajectories by their unique cut state yields
$\sum_{s\in\mathcal C}F^0(s)=\sum_{\tau}m^0(\tau)=Z^\star(q)$. For a fixed state,
$F^0(s)=\sum_{e\in\Out(s)}F^0(s)\PF^0(e\mid s)$, so dividing by $Z^\star(q)$
gives
$\sum_{e\in\Out(s)}\varphi^0(e)=F^0(s)/Z^\star(q)$. Finally, if each counted
skill edge has one and only one skill label, the skill sums form a partition of
$\sum_{e\in\mathcal E_{\mathrm{skill}}}\varphi^0(e)$; after dividing by this
same positive denominator, $\sum_u\Psi^0(u)=1$, the normalisation of the skill share in \cref{eq:share}.
\end{proof}

\paragraph{Utility: a separate, signed quantity.}
For utility we compare an event against a
\emph{fixed} reference distribution $\nu$, uniform over the legal events at $s$,
\begin{equation}
  A^{\mathrm{util}}_e
  \;=\;
  \widehat{\E}\bigl[\Rt\mid s,e\bigr]
  -\E_{e'\sim\nu(\cdot\mid s)}\bigl[\widehat{\E}[\Rt\mid s,e']\bigr],
  \label{eq:util}
\end{equation}
and discount it by the edge-level balance violation $\dbres_e=\dbres_{t-1:t}$,
\begin{equation}
  \widetilde{A}_e \;=\; A^{\mathrm{util}}_e\,
  \exp\!\bigl(-|\dbres_e|/\tau_c\bigr).
  \label{eq:discount}
\end{equation}
Here $\widehat{\E}[\Rt\mid s,e]$ estimates
$Q^{\pi_{\mathrm{eval}}}(s,e)$ under a fixed continuation policy
$\pi_{\mathrm{eval}}$.
\chg{New. A reference distribution that is the \emph{current} policy drifts as
the policy collapses; $\nu$ is frozen. The discount reduces the influence of an
unconverged flow model on proposal ordering; the verifier and validation gate
still determine commitment.}

\section{Evidence, Gating, and Skill Evolution}
\label{app:evolve}
\subsection{Verifier-Gated Bayesian Calibration}
\label{sec:bayes}

Each invocation carries a context feature vector
\begin{equation}
  z=\bigl[\text{context class},\ \text{failure mode},\
          \text{token bucket},\ \text{horizon bucket}\bigr].
  \label{eq:z}
\end{equation}
Context features use only information available before the invocation; a
failure mode observed afterwards enters later invocations as a lagged feature.

\paragraph{Evidence.}
Labels come from independent verifiers (unit tests, schema validation,
assertions, execution status) with confidence $c^{\mathrm{ver}}_e\in[0,1]$.
For each skill/context cell define the verifier-record index set, which collects the
verifier records of skill $u$ observed in context $z$,
\begin{equation}
  \mathcal E^{\mathrm{ver}}_{u,z}
  =\{e\in\mathcal E^{\mathrm{ver}}:\ u(e)=u,\ z(e)=z\}.
\end{equation}
Only these verifier records enter the gate posterior:
\begin{align}
  \alpha^{\mathrm{ver}}_{u,z} &= \alpha_{0,u}
    + \sum_{e\in\mathcal{E}^{\mathrm{ver}}_{u,z}} c^{\mathrm{ver}}_e y_e,
  \label{eq:alpha}\\
  \beta^{\mathrm{ver}}_{u,z} &= \beta_{0,u}
    + \sum_{e\in\mathcal{E}^{\mathrm{ver}}_{u,z}} c^{\mathrm{ver}}_e (1-y_e).
  \label{eq:beta}
\end{align}
A soft label $\widetilde y_e\in[0,1]$ derived from $\widetilde A_e$ enters a
separate proposal score $q^{\mathrm{soft}}_{u,z}$ that orders candidates and
stays outside the gate posterior $(\alpha^{\mathrm{ver}},\beta^{\mathrm{ver}})$ of \cref{eq:alpha,eq:beta}.
\chg{L116--122: the original weighted the conjugate update by the flow
magnitude $w_e\propto F(H^e_t)$. Since $F$ was an exponential of a cumulative
sum it was unbounded, so $\alpha+\beta$ carried no sample-size meaning and a
single trajectory could dominate the posterior. Flow now plays a different
role, below.}

\begin{remark}[Flow allocates the verification budget]
\label{rem:acq}
$\varphi^0(e)$ decides \emph{which} invocations are worth verifying; verifier
confidence decides \emph{how much} an
observation counts as evidence. Keeping the two roles separate is what makes the
posterior interpretable.
\end{remark}

\begin{remark}[Effective sample size]
\label{rem:ess}
With fractional weights, the quantity that behaves like a sample count is Kish's effective sample size \citep{kish1965survey},
\begin{equation}
  N_{\mathrm{eff}}=\frac{\bigl(\sum_e w_e\bigr)^{2}}{\sum_e w_e^{2}}
  \quad(\text{zero without records}),
  \label{eq:ess}
\end{equation}
not $\sum_e w_e$: one weight of size $N$ with the rest near zero has
$\sum_e w_e=N$ but $N_{\mathrm{eff}}\approx1$.
\end{remark}

\paragraph{Hierarchical prior and credible bounds.}
Contexts are pooled towards the skill-level reliability $\mu_u$ with strength
$\kappa_u$, using the same prior parameters as the verifier gate,
$\alpha_{0,u}=\kappa_u\mu_u$ and
$\beta_{0,u}=\kappa_u(1-\mu_u)$:
\begin{equation}
  p_{u,z}\ \sim\ \mathrm{Beta}\bigl(\kappa_u\mu_u,\ \kappa_u(1-\mu_u)\bigr),
  \label{eq:hier}
\end{equation}
and decisions use Beta quantiles rather than a normal approximation:
\begin{equation}
  \mathrm{LCB}_{u,z}=Q_{a}(\alpha^{\mathrm{ver}}_{u,z},\beta^{\mathrm{ver}}_{u,z}),
  \qquad
  \mathrm{UCB}_{u,z}=Q_{1-a}(\alpha^{\mathrm{ver}}_{u,z},\beta^{\mathrm{ver}}_{u,z}).
  \label{eq:bounds}
\end{equation}
With fractional confidences these bounds are credible quantiles of a power
posterior (\cref{prop:power-posterior}). The $-$~Credible bounds ablation of
\Cref{tab:ablation} replaces them by the posterior mean
$\alpha^{\mathrm{ver}}_{u,z}/(\alpha^{\mathrm{ver}}_{u,z}+\beta^{\mathrm{ver}}_{u,z})$.

\subsection{Phase Transition and Skill Evolution}
\label{sec:phase}

\paragraph{A partition-function-free stagnation statistic.}
For a fixed query $q$, fixed graph, policy, and trajectory law,
the source log-flow $\log Z_\theta(q)$ shifts every complete-trajectory residual $\dbres_{0:T}(\tau)$ by the same constant, hence the residual variance
\begin{equation}
  V_q \;=\; \Var_{\tau\mid q}\bigl[\dbres_{0:T}(\tau)\bigr]
  \label{eq:vardiag}
\end{equation}
is invariant to that source-log-flow shift. It is estimated from independent
rollouts of each query, and queries are aggregated with a random-effects model.
\chg{L139--143: the original trigger used the relative decrease of
$\overline{\Delta^2}$. That statistic is contaminated by how well $Z_\theta$ has
been fitted and is not comparable across library versions.}

\paragraph{Trigger.}
Let $\bar V_w$ be $V_q$ aggregated over the held-out validation query set (\Cref{app:full-results}) at training step
$w$, $b$ the slope of a linear fit to $\bar V_w$ over a window $W$, and $\mathrm{CI}(b)$ its confidence interval at a fixed level. A phase boundary is declared iff all three hold:
\begin{equation}
  \mathrm{CI}(b)\subseteq[-\eps_b,\,\eps_b],
  \qquad
  \frac{\bar V_{w-W}-\bar V_{w}}{\max\{\bar V_{w-W},\,v_{\min}\}}\in[0,\gamma_{\mathrm{var}}),
  \qquad
  \Delta H_{\mathrm{norm}} < -h_0 ,
  \label{eq:trigger}
\end{equation}
where $\eps_b>0$ is the slope tolerance, $v_{\min}>0$ a numerical floor,
$\gamma_{\mathrm{var}}>0$ the plateau tolerance, $h_0>0$ a margin, and
$\Delta H_{\mathrm{norm}}$ the change over the window of the normalised skill
entropy $H_{\mathrm{norm}}=H(\mathcal{S}\mid q,z)/\log|\mathcal{S}_{\mathrm{active}}|$,
with $H(\mathcal{S}\mid q,z)=-\sum_u p_{u\mid q,z}\log p_{u\mid q,z}$ and
$p_{u\mid q,z}\propto n_{u,q,z}^{\mathrm{call}}$ computed from the held-out
rollouts in that query/context cell. The first condition is an equivalence test
that places the whole interval inside a tolerance band, the lower bound of zero
in the second excludes divergence, and the normalised entropy keeps changes in
library size from masquerading as saturation.
\chg{L139--142: ``the slope is not significantly negative'' is absence of
evidence and is satisfied by any under-powered test. \cref{eq:trigger} uses an
equivalence criterion (the whole interval inside a tolerance band), a lower
bound of $0$ on the relative change to exclude divergence, and a normalised
entropy so that library size changes do not masquerade as saturation.}

\paragraph{Evolution operator.}
The library is refined by
\begin{equation}
  \Slib{k+1} \;=\; \Phi\bigl(\Slib{k};\ \Psi^0(u),\ n_u^{\mathrm{call}},\ \mathrm{LCB}_{u,z},\
  \mathrm{UCB}_{u,z},\ \widetilde{A}_e\bigr),
  \label{eq:phi}
\end{equation}
whose actions live at three levels: single skills, skill pairs and context
clusters.

\emph{Single-skill decisions}, evaluated in this order for each skill $u$:
\begin{enumerate}[label=(\arabic*),leftmargin=*]
  \item \textbf{Defer / Explore}: $N_{\mathrm{eff},u}<n_{\min}$. No structural
        change; allocate verification and exploration budget instead, since the effective sample size of $u$ is below $n_{\min}$.
  \item \textbf{Split}: with sufficient support in each context,
        $\Prob\bigl(\max_z p_{u,z}-\min_z p_{u,z}>\theta_h\mid\mathcal{D}\bigr)
        >1-a$.
  \item \textbf{Refine}: globally sound but locally confirmed weak,
        $\mathrm{LCB}^{\mathrm{glob}}_{u}\ge\theta_{\mathrm{mid}}$ and
        $\exists z:\ \mathrm{UCB}_{u,z}<\theta_{\mathrm{low}}$.
  \item \textbf{Retain}: confirmed reliable at the skill level, $\mathrm{LCB}^{\mathrm{glob}}_{u}\ge\theta_{\mathrm{high}}$.
  \item \textbf{Prune}: $\max_z \mathrm{UCB}_{u,z}<\theta_{\mathrm{low}}$ and
        $\overline{\widetilde{A}}(u)\le 0$ and $u$ provides no unique coverage.
\end{enumerate}
\emph{Pairwise}: \textbf{Consolidate} is a redundancy test on a skill pair: near
equivalence of invocation contexts, outputs and verifier outcomes, with no
coverage loss on removal; downstream reward and cost are checked by the paired validation below.
\emph{Cluster-level}: \textbf{Generate} requires a context cluster with
sufficient support, concentrated failure mass and no existing skill whose
$\mathrm{UCB}$ is adequate there, and the new skill must be checkable by the
verifiers.
Here $N_{\mathrm{eff},u}$ is the Kish effective sample size \eqref{eq:ess} over
the verifier records of skill $u$, and sufficient support means an effective
sample size of at least $n_{\min}$ for the skill, context or cluster concerned;
$\mathrm{LCB}^{\mathrm{glob}}_u$ is the lower credible bound of the skill-level
posterior pooled over contexts; $u$ provides unique coverage when some context it
serves is served by no other skill; and $\overline{\widetilde A}(u)$ aggregates
$\widetilde A_e$ over the invocations of $u$ and serves for proposal ordering
and as the utility veto of Prune. The thresholds satisfy
$0\le\theta_{\mathrm{low}}<\theta_{\mathrm{mid}}<\theta_{\mathrm{high}}\le1$.
\chg{L150--157: the five original bullets mixed three granularities and
overlapped (high flow with low LCB matched both Refine and Split). Note that
Refine now keys on a \emph{low UCB} locally: a low LCB alone only means the
evidence is thin.}

\paragraph{Versioning, rollback, hysteresis.}
Each application of $\Phi$ yields $\Gph{k+1}$ and is accepted only after paired
rollouts on the held-out validation query set with a non-inferiority margin on success,
tempered terminal reward, token cost and latency, assessed by a one-sided test per metric against that margin
\citep{schuirmann1987comparison}; otherwise it is rolled back. A cooldown prevents
a skill from being re-edited within the cooldown window through
Split--Refine--Prune.
\chg{L158: ``the partition function $Z_\theta(q)$ is re-initialised'' is
replaced by a rule: if the terminal support and reward are unchanged, $Z_\theta$
is warm-started; only a change of reachable outcomes calls for re-calibration,
and even then the old value initialises it.}

\begin{definition}[Verifier-only acceptance gate]
\label{def:accept}
Let $\mathcal C_k$ be the candidate edits proposed or ranked from flow
diagnostics, signed utility, and retained records at phase $k$. Define
$\operatorname{Acc}_k(d)=1$ only if (i) every structural predicate for $d$ is
computed from the verifier posterior
$(\alpha^{\mathrm{ver}},\beta^{\mathrm{ver}})$ and its support threshold
$n_{\min}$, (ii) the candidate passes the paired held-out
non-inferiority check, and (iii) the version, cooldown and coverage predicates
hold. The soft score $q^{\mathrm{soft}}$ and $\widetilde A_e$ rank $\mathcal C_k$, and Prune
additionally requires the utility veto $\overline{\widetilde A}(u)\le0$ for the pruned skill $u$.
\end{definition}

\begin{proposition}[Gate and rollback invariants]
\label{prop:gate}
If $\operatorname{Acc}_k$ is applied after candidate generation and before a
library version is committed, then every committed edit
has verifier-only structural eligibility and paired validation, while a failed
validation or failed gate leaves the active library equal to $\Slib{k}$.
\end{proposition}
\begin{proof}
By definition, a candidate can enter the commit branch only when
$\operatorname{Acc}_k(d)=1$, which includes the verifier-only predicate and the
paired non-inferiority check. If either check fails, the transition takes the
defer or rollback branch, whose returned library is defined to be the previous
version $\Slib{k}$. Since $q^{\mathrm{soft}}$ and $\widetilde A_e$ cannot supply the verifier-only
eligibility condition, changing them while holding the verifier posterior and
validation outcomes fixed cannot turn an ineligible candidate into an eligible
one. Reward-derived signals can therefore block an edit through the utility
veto of Prune but never make one eligible for the commit branch.
\end{proof}

\begin{proposition}[Phase-state recursion]
\label{prop:recursion}
Let $X_k$ contain the active library and graph version, trainable policy and flow
parameters, retained verifier statistics, and all persistent bookkeeping read by
the next transition. Under a fixed outer protocol $\lambda$ and phase randomness
$\zeta_k$, the completed phase transition has the form
\begin{equation}
  X_{k+1}=\mathcal U_\lambda(X_k;\zeta_k),\qquad
  r_k=\mathcal L_\lambda(X_k,X_{k+1};\zeta_k),\qquad
  \mathcal H^{\mathrm{imp}}_{k+1}
    =\mathcal H^{\mathrm{imp}}_k\frown[r_k].
  \label{eq:state-transition}
\end{equation}
Consequently $\Xi_k=(X_k,\mathcal H^{\mathrm{imp}}_k)$, which pairs $X_k$ with the log of edit decisions and validation outcomes, obeys the well-defined
recurrence $\Xi_{k+1}=\mathfrak R_\lambda(\Xi_k;\zeta_k)$.
\end{proposition}
\begin{proof}
The fixed protocol maps a completed phase state and its random input to the next
state. The record map $\mathcal L_\lambda$ is total on completed transitions, so
sequence concatenation defines the next audit log. Pairing the two updates gives
$\mathfrak R_\lambda$. Because $X_k$ contains every field the next transition
reads, this map is a function of the phase state $\Xi_k$ for fixed phase randomness $\zeta_k$.
\end{proof}

\subsection{Approximate Balance and Abstraction Error}
\label{sec:scope}

\paragraph{Reasoning-conditioned events.}
The implemented residual scores $\log\PF(e_t\mid s_{t-1},r_t)$ with the sampled
reasoning text as context; free-text arguments of a skill call are treated the
same way, so $\PF(\mathrm{args}\mid u,s)$ scores only the structured arguments. Sampling and
scoring then use the same conditional policy, so the balance identities apply to
it (\cref{rem:lengthnorm}). When $|\dbres_{0:T}|\le\varepsilon$
holds for every sampled reasoning text, the terminal-distribution bound of \cref{prop:eps} applies to the marginal
policy as well, because the marginal trajectory probability averages the
conditional trajectory probabilities, each of which satisfies the residual bound.

\begin{proposition}[Residual to terminal distribution]
\label{prop:eps}
Let $\widehat Z=F_\psi(s_0)=Z_\theta(q)$ be the learned source flow and
$Z^\star=\sum_{x\in\mathcal X}\Rt(x)$ the true reward partition function. On
the finite reachable DAG, assume that the forward and backward policies have
positive support on every legal edge and that
$|\dbres_{0:T}(\tau)|\le\varepsilon$ for every complete legal trajectory
$\tau$ (hence for every terminal-support fibre). Then for every terminal
$x$,
\begin{equation}
  e^{-\varepsilon}\;\le\;\frac{\widehat Z\,P_F^{\mathrm{term}}(x\mid q)}{\Rt(x)}\;\le\;e^{\varepsilon},
  \qquad
  e^{-\varepsilon}\;\le\;\frac{\widehat Z}{Z^\star}\;\le\;e^{\varepsilon},
  \qquad
  e^{-2\varepsilon}\;\le\;\frac{Z^\star P_F^{\mathrm{term}}(x\mid q)}{\Rt(x)}\;\le\;e^{2\varepsilon}.
  \label{eq:epsbound}
\end{equation}
\end{proposition}
\begin{proof}
$\dbres_{0:T}(\tau)=\log\bigl[\widehat Z\PF(\tau)/(\Rt(x_\tau)\PB(\tau\mid
x_\tau))\bigr]$. Exponentiating the residual bound and summing over trajectories
ending at $x$, while using $\sum_{\tau\to x}\PB(\tau\mid x)=1$, gives the first
inequality. Summing the first inequality after multiplying by
$\Rt(x)/\widehat Z$ over all terminal states gives
$e^{-\varepsilon}\le\widehat Z/Z^\star\le e^{\varepsilon}$. Multiplying the
first inequality by $Z^\star/\widehat Z$, which the second inequality bounds, yields the third inequality, with the exponent $2\varepsilon$ in place of $\varepsilon$.
\end{proof}

\begin{proposition}[Abstraction error propagation]
\label{prop:absterr}
Let $\sigma$ be the abstraction map and $\pi$ a common abstract policy. For a
concrete history $H$ with $\sigma(H)=s$, let
$K_{H,t}^{\sigma,\pi}(\cdot\mid s)$ be the history-level one-step
terminal-or-successor kernel, and let $\overline K_t(\cdot\mid s)$ be the
adopted quotient kernel. Assume a common initial abstract distribution, terminal
space, policy and termination convention, and
\begin{equation}
  \sup_{H:\,\sigma(H)=s}
  \operatorname{TV}\!\left(K_{H,t}^{\sigma,\pi}(\cdot\mid s),
  \overline K_t(\cdot\mid s)\right)\le \varepsilon_t
  \quad\text{for every reachable }s\text{ at decision step }t.
  \label{eq:abst-kernel}
\end{equation}
If terminated paths are padded by a common absorbing transition and all paths
have at most $T$ decision steps, then the terminal marginals of the history-level
and quotient processes differ by at most
$\min\{1,\sum_{t=0}^{T-1}\varepsilon_t\}$ in total variation. With
$\varepsilon_t=0$ at every decision step, the kernel bound of \cref{eq:abst-kernel} is the model-irrelevance condition of
\citet{li2006towards}, and the terminal marginals of the history-level and quotient processes then coincide.
\end{proposition}
\begin{proof}
Project the history-level process to the abstract state. Conditional on a
projected state, its one-step kernel is the mixture of the history-level kernels
over the conditional distribution of concrete histories. Convexity of total
variation and \cref{eq:abst-kernel} bound the distance between this mixture and
$\overline K_t$ by $\varepsilon_t$. A stepwise hybrid argument, or induction
using contraction of Markov kernels, then gives the sum of the per-step bounds;
the outer minimum uses $\operatorname{TV}\le1$.
\end{proof}

\begin{remark}[Reward measurability of merged terminals]
\Cref{ass:quotient}(c) makes $\Rt(x)$ well defined on every merged terminal, so
the terminal anchor of \cref{eq:anchor} applies there; the abstraction map $\sigma$ meets this terminal-reward consistency by keeping in $s$ every feature
of the history that the reward reads, which makes the reward $\sigma$-measurable.
\end{remark}

\section{Proof Details}
\label{app:proof-details}

This section proves the flow identities used by the method text. The exact
results characterise the zero-residual optimum, and \cref{prop:eps} bounds the
deviation when residuals are small. The first subsection derives the path
factorisation and the balance identities, the second builds the compatible
reference flow and the rollout estimator of the flow share, and the third treats
the verifier posteriors behind the credible bounds and the acceptance gate of
\cref{def:accept}.

\subsection{Path factorisation and balance identities}

\begin{lemma}[Path factorisation]
\label{lem:path-factor}
Under \cref{ass:quotient,ass:markov,ass:finite,ass:det}, let
$\tau=(s_0,e_1,s_1,\ldots,e_T,s_T)$ be a legal trajectory ending at $x=s_T$.
For any normalized forward policy $\PF$ and backward policy $\PB$ with positive
support on the edges of $\tau$,
\begin{equation}
  \PF(\tau)=\prod_{t=1}^{T}\PF(e_t\mid s_{t-1}),
  \qquad
  \PB(\tau\mid x)=\prod_{t=1}^{T}\PB((s_{t-1},e_t)\mid s_t).
  \label{eq:path-factor}
\end{equation}
\end{lemma}
\begin{proof}
The Markov assumption makes the conditional law of the next event depend only on
$(s_{t-1},q,k)$, and \cref{ass:det} makes each successor a function of the
current state and event, so the chain rule reduces to the first product. Running
the normalized in-edge policy backward from the terminal $x=s_T$ gives the second product of \cref{eq:path-factor}.
\end{proof}

\begin{lemma}[Residual additivity]
\label{lem:res-add}
For $0\le i<\ell<j\le T$, the residual in \cref{eq:residual} satisfies
\begin{equation}
  \dbres_{i:j}=\dbres_{i:\ell}+\dbres_{\ell:j}.
  \label{eq:res-add}
\end{equation}
Consequently, zero residual on every sub-trajectory implies both detailed balance
on every edge and trajectory balance on every complete trajectory of the DAG.
\end{lemma}
\begin{proof}
Expand the two terms on the right of \cref{eq:res-add}. The intermediate terms
$\log\Fpsi(s_\ell)$ and $-\log\Fpsi(s_\ell)$ cancel, as do no other terms,
leaving exactly the sums from $i+1$ through $j$. Choosing $j=i+1$ gives detailed
balance on every edge, and choosing $(i,j)=(0,T)$ gives trajectory balance on every complete trajectory. The converse does
not hold from trajectory balance alone: internal flows can change while
preserving the endpoint product.
\end{proof}

\begin{proposition}[Sub-trajectory objective and exact balance]
\label{prop:subtb-exact}
Let $\mu$ be the distribution over complete legal trajectories used for the
population objective and write
$\mathcal J_\mu=\mathbb E_{\tau\sim\mu}[\LTB(\tau)]$. Assume
$\omega_{j-i}>0$ for every included pair and
$\Pr_{\mu}(e\in\tau)>0$ for every constrained edge. Then
$\mathcal J_\mu=0$ if and only if every included residual is zero on the
support of $\mu$. If all contiguous pairs are included
and the support covers every legal edge, this is equivalent to the collection of
detailed-balance identities on the DAG together with the terminal anchor of \cref{eq:anchor}.
\end{proposition}
\begin{proof}
The integrand is a finite sum of nonnegative terms. Its expectation is zero
only if every term is zero $\mu$-almost surely. The positive edge-visitation
condition makes each constrained edge occur with positive probability, so each
included residual is zero on that edge's support. The first implication of
\cref{lem:res-add} gives detailed balance and trajectory identities. Conversely,
those identities make every included residual $\dbres_{i:j}$ in the objective of \cref{eq:flow-objective} zero.
\end{proof}

\begin{remark}[From exact to approximate balance] A finite-sample optimizer leaves small nonzero residuals; when every complete-trajectory residual is at most $\varepsilon$ in magnitude, \cref{prop:eps} bounds the resulting deviation of the terminal law in \cref{eq:epsbound}.
\end{remark}

\subsection{Compatible reference flows}

\begin{proposition}[Backward reconstruction of a compatible triple]
\label{prop:compatible-triple}
On a finite DAG in which every reachable node lies on a source-to-terminal path,
with positive terminal rewards and a normalized full-support backward policy
$\PB^0$ on every nonempty in-edge set, define the reference state and edge flows recursively
\begin{equation}
\begin{aligned}
  F^0(x)&=\Rt(x) && (x\in\mathcal X),\qquad
  f^0(s,e,s')=F^0(s')\PB^0((s,e)\mid s'),\\
  F^0(s)&=\sum_{(s,e,s')\in\Out(s)}f^0(s,e,s')
  && (s\notin\mathcal X).
\end{aligned}
  \label{eq:reference-recurrence}
\end{equation}
Then $Z^\star=F^0(s_0)=\sum_x\Rt(x)$ and
$\PF^0(e\mid s)=f^0(s,e,s')/F^0(s)$ form a compatible flow triple
$(F^0,\PF^0,\PB^0)$. It is unique for the chosen $\PB^0$ and terminal anchor.
\end{proposition}
\begin{proof}
The recurrence is evaluated in decreasing rank. Terminal values are fixed by the
terminal anchor. Since every reachable node lies on a terminal path and all terminal
rewards and backward probabilities are positive, induction gives $F^0(s)>0$ at
every reachable node. Once all children of $s$ have been assigned flow values, every
outgoing edge mass is fixed by $\PB^0$, so their sum fixes $F^0(s)$ and division by it
fixes the normalized forward probabilities. For each non-root node $s'$,
normalization of the backward policy $\PB^0$ gives
$\sum_{(s,e)\in\Inn(s')}f^0(s,e,s')=F^0(s')$. Summing this incoming-flow
identity over all non-root states and the outgoing-flow definition over all
nonterminal states cancels the mass of every internal edge, leaving
$F^0(s_0)=\sum_{x\in\mathcal X}F^0(x)=\sum_x\Rt(x)$. The same induction over decreasing rank shows
uniqueness. The construction also gives
$F^0(s)\PF^0(e\mid s)=F^0(s')\PB^0((s,e)\mid s')$, which is the required
compatibility identity.
\end{proof}

\begin{corollary}[Terminal law under an exact compatible triple]
\label{cor:terminal-law}
If a forward policy is generated by \cref{prop:compatible-triple}, then its
terminal marginal is
\begin{equation}
  P_{F}^{0,\mathrm{term}}(x\mid q)=\frac{\Rt(x)}{Z^\star}.
  \label{eq:terminal-law}
\end{equation}
The terminal law is invariant to the choice of full-support $\PB^0$, while the
internal state and edge flows generally are not, as in the two-path example of \cref{prop:pb}.
\end{corollary}
\begin{proof}
For every trajectory ending at $x$, compatibility gives
$\PF^0(\tau)=\Rt(x)\PB^0(\tau\mid x)/Z^\star$. Summing over trajectories and
using normalization of the backward trajectory distribution gives
\cref{eq:terminal-law}. Internal flows contain the factors
$\PB^0(\tau\mid x)$ before they are summed, so two backward policies can induce
different internal allocations while preserving the same terminal marginal.
\end{proof}

\begin{remark}[Reference readout versus learned training triple]
Equation~\eqref{eq:share} is a flow diagnostic only when all of its factors come
from one compatible triple; pairing a learned $P_{F,\theta}$ with a state flow
built from another $\PB^0$ need not conserve flow. \RtwoFlow therefore estimates
the reference readout from the trained policy's rollouts with the residual
weights of \cref{cor:share-estimator}, which needs no reconstruction.
\end{remark}

\begin{lemma}[Occupancy identity]
\label{lem:occupancy}
For a compatible positive flow on the fixed rooted, terminal-complete DAG,
\begin{equation}
  \Pr_{\PF^0}(s\in\tau)=\frac{F^0(s)}{Z^\star},
  \qquad
  \Pr_{\PF^0}((s,e,s')\in\tau)=
  \frac{F^0(s)\PF^0(e\mid s)}{Z^\star}.
  \label{eq:occupancy}
\end{equation}
\end{lemma}
\begin{proof}
The source is visited with probability one. In increasing rank, incoming-edge
events to a fixed state are mutually exclusive, so its visitation probability
satisfies
\[
 v(s')=\sum_{(s,e)\in\Inn(s')}v(s)\PF^0(e\mid s)
      =\frac{\sum_{(s,e)\in\Inn(s')}F^0(s)\PF^0(e\mid s)}{Z^\star}
      =\frac{F^0(s')}{Z^\star},
\]
where the last equality is flow conservation. Multiplying $v(s)$ by the
conditional edge probability gives the edge identity. Terminal completeness
makes these visitation probabilities equal to sums over the complete trajectories containing the state or edge.
\end{proof}

\begin{proposition}[Flow-share normalization]
\label{prop:share-normalization}
Fix a finite compatible reference-flow triple
$(F^0,\PF^0,\PB^0)$ with source flow $F^0(s_0)=Z^\star>0$. Let
$\mathcal E_u$ be the set of counted edges labelled by skill $u$, assume that
the sets $\{\mathcal E_u\}_u$ form a disjoint partition of
$\mathcal E_{\mathrm{skill}}$, and define
\begin{equation}
  \varphi^0(e)=\frac{F^0(s)\PF^0(e\mid s)}{Z^\star}
  \quad\text{for }e=(s\!\to\!s').
  \label{eq:edge-visit-mass-detailed}
\end{equation}
For a complete trajectory $\tau$, let
\begin{equation}
  I_e(\tau)=\mathbf 1\{e\in\tau\},
  \qquad
  N_u(\tau)=\sum_{e\in\mathcal E_u}I_e(\tau),
  \qquad
  T_{\mathrm{skill}}(\tau)
    =\sum_{e\in\mathcal E_{\mathrm{skill}}}I_e(\tau).
  \label{eq:skill-counts}
\end{equation}
Then
\begin{equation}
  \sum_{e\in\mathcal E_u}\varphi^0(e)
     =\mathbb E_{\tau\sim\PF^0}[N_u(\tau)],
  \qquad
  M_{\mathrm{skill}}
     =\sum_{e\in\mathcal E_{\mathrm{skill}}}\varphi^0(e)
     =\mathbb E_{\tau\sim\PF^0}[T_{\mathrm{skill}}(\tau)].
  \label{eq:share-count-identities}
\end{equation}
If $M_{\mathrm{skill}}>0$, the normalized skill share is therefore
\begin{equation}
  \Psi^0(u)
  =\frac{\mathbb E_{\PF^0}[N_u(\tau)]}
         {\mathbb E_{\PF^0}[T_{\mathrm{skill}}(\tau)]}
  =\frac{\sum_{e\in\mathcal E_u}\varphi^0(e)}{M_{\mathrm{skill}}},
  \qquad
  \sum_u\Psi^0(u)=1.
  \label{eq:share-normalization}
\end{equation}
In general this is a ratio of expectations, not
$\mathbb E[N_u(\tau)/T_{\mathrm{skill}}(\tau)]$.
\end{proposition}
\begin{proof}
Because the graph is acyclic, a complete trajectory traverses any fixed edge at
most once. Compatible flow balance gives
\begin{equation}
  F^0(s)\PF^0(e\mid s)
  =\sum_{\tau:\,e\in\tau}F^0(s_0)\PF^0(\tau).
  \label{eq:edge-flow-path-sum}
\end{equation}
By \cref{lem:occupancy}, dividing by $F^0(s_0)=Z^\star$ yields
\begin{equation}
  \varphi^0(e)
  =\sum_{\tau}\PF^0(\tau)I_e(\tau)
  =\mathbb E_{\tau\sim\PF^0}[I_e(\tau)].
  \label{eq:edge-visit-probability}
\end{equation}
All sums are finite, so linearity of expectation permits exchanging the edge
sum and the trajectory expectation. This proves both identities in
\cref{eq:share-count-identities}. Since the skill-labelled edge sets form a
partition, $T_{\mathrm{skill}}(\tau)=\sum_uN_u(\tau)$ pointwise. Taking
expectations and dividing by the common positive denominator $M_{\mathrm{skill}}$ of \cref{eq:share-count-identities} proves
\cref{eq:share-normalization}.
\end{proof}

\begin{corollary}[Backward-policy independence under path-invariant counts]
\label{cor:share-pb-independence}
Suppose that for every terminal node $x$ there are deterministic functions
$n_u(x)$ and $t_{\mathrm{skill}}(x)$ such that every complete trajectory ending
at $x$ satisfies
\begin{equation}
  N_u(\tau)=n_u(x),
  \qquad
  T_{\mathrm{skill}}(\tau)=t_{\mathrm{skill}}(x).
  \label{eq:path-invariant-counts}
\end{equation}
Then every full-support compatible backward policy $\PB^0$ induces the same
normalized skill share:
\begin{equation}
  \Psi^0(u)
  =\frac{\sum_x \Rt(x)n_u(x)}
       {\sum_x \Rt(x)t_{\mathrm{skill}}(x)},
  \label{eq:share-pb-independent}
\end{equation}
provided the denominator of \cref{eq:share-pb-independent} is positive.
\end{corollary}
\begin{proof}
Compatibility gives the terminal law
$\PF^0(x)=\Rt(x)/Z^\star$, independently of the chosen full-support backward
policy. Under \cref{eq:path-invariant-counts}, conditioning on $x$ removes all
remaining path dependence:
\begin{equation}
  \mathbb E_{\PF^0}[N_u(\tau)]
  =\sum_x\frac{\Rt(x)}{Z^\star}n_u(x),
  \qquad
  \mathbb E_{\PF^0}[T_{\mathrm{skill}}(\tau)]
  =\sum_x\frac{\Rt(x)}{Z^\star}t_{\mathrm{skill}}(x).
\end{equation}
Substitution into \cref{eq:share-normalization} cancels $Z^\star$ and yields
\cref{eq:share-pb-independent}.
\end{proof}

\begin{corollary}[Skill share on the shared-state graph and its estimator]
\label{cor:share-estimator}
Let every edge commit one atomic event labelled by one skill $u\in\Slib{k}$ (\cref{def:edges}).
\begin{enumerate}[label=(\roman*),leftmargin=*]
  \item Every complete trajectory ending at $x$ satisfies
        \cref{eq:path-invariant-counts}, with $n_u(x)$ the number of events of
        $E_x$ that invoke $u$ and $t_{\mathrm{skill}}(x)=\sum_u n_u(x)$, so
        $\Psi^0(u)$ is given by \cref{eq:share-pb-independent} for every
        full-support compatible backward policy $\PB^0$, as in \cref{cor:share-pb-independence}.
  \item Let $\tau_1,\dots,\tau_N$ be rollouts of query $q$ drawn independently
        from a forward policy with positive support on every legal edge, let
        $\PB$ be any normalized full-support backward policy, and let
        $w_i=\exp(-\dbres_{0:T}(\tau_i))$ be computed from them. Then
        \begin{equation}
          \widehat\Psi^0_N(u)
          =\frac{\sum_{i}w_i\,n_u(x_{\tau_i})}{\sum_{i}w_i\,t_{\mathrm{skill}}(x_{\tau_i})}
          \;\longrightarrow\;\Psi^0(u)\quad\text{almost surely as }N\to\infty,
          \label{eq:share-estimator}
        \end{equation}
        whatever the values of the source flow $\Fpsi(s_0)$ and the domain bias $b_d$.
  \item With unit weights and $|\dbres_{0:T}|\le\varepsilon$ on every complete
        legal trajectory, the limit of $\widehat\Psi^0_N(u)$ in \cref{eq:share-estimator} lies in
        $[e^{-4\varepsilon},e^{4\varepsilon}]\,\Psi^0(u)$.
\end{enumerate}
The same weights estimate the edge share,
$\widehat\varphi^0_N(e)=\sum_i w_i\,h_{B^0}(e\mid x_{\tau_i})/\sum_i w_i$, where
$h_{B^0}(e\mid x)$ is the probability that a backward walk under $\PB^0$ from $x$
traverses $e$.
\end{corollary}
\begin{proof}
(i) A state records its committed events, the source has none, and each edge
adds one event, so a complete trajectory ending at $x$ adds every event of $E_x$
exactly once and its skill counts depend on $x$ alone;
\cref{cor:share-pb-independence} then applies.
(ii) As in the proof of \cref{prop:eps},
$w_i=c_q\,\Rt(x)\PB(\tau_i\mid x)/\PF(\tau_i)$ with $x=x_{\tau_i}$ and a constant
$c_q>0$ that collects $\Fpsi(s_0)$ and $b_d$ for the query. For any function
$g$ of the terminal,
$\E_{\PF}[w\,g(x_\tau)]=c_q\sum_\tau\Rt(x_\tau)\PB(\tau\mid x_\tau)g(x_\tau)
=c_q\sum_x\Rt(x)g(x)$, because $\sum_{\tau\to x}\PB(\tau\mid x)=1$. Taking
$g=n_u$ and $g=t_{\mathrm{skill}}$, the strong law of large numbers applies to
the numerator and the denominator, $c_q$ cancels, and the ratio converges to
\cref{eq:share-pb-independent}. When the forward term conditions on sampled
reasoning text, $\PF(\tau_i)$ is the event probability given that text; the
text is drawn from its own distribution, so summing over it leaves the same
expectation.
(iii) With unit weights the limit is
$\sum_x P_F^{\mathrm{term}}(x\mid q)\,n_u(x)/\sum_x P_F^{\mathrm{term}}(x\mid q)\,t_{\mathrm{skill}}(x)$.
By \cref{prop:eps}, $P_F^{\mathrm{term}}(x\mid q)$ lies within a factor
$e^{\pm2\varepsilon}$ of $\Rt(x)/Z^\star$, so the numerator and the denominator
each lie within that factor of their targets and the ratio within
$e^{\pm4\varepsilon}$.
For the edge share, take $g(x)=h_{B^0}(e\mid x)$ and $g\equiv1$; the limit is
$\sum_x\Rt(x)h_{B^0}(e\mid x)/Z^\star=\varphi^0(e)$ by \cref{prop:flow} applied
to the reference triple and the compatibility identity
$F^0(s)\PF^0(e\mid s)=F^0(s')\PB^0((s,e)\mid s')$.
\end{proof}

\begin{remark}[Precision of the share estimate]
\label{rem:share-ess}
The Kish effective sample size \eqref{eq:ess} of $w_1,\dots,w_N$ measures the
precision of \cref{eq:share-estimator}. The log-weights have variance $V_q$, so
if they are roughly Gaussian the effective sample size is about $Ne^{-V_q}$ and
the estimate sharpens as $V_q$ falls.
\end{remark}

\subsection{Verifier posteriors and decision invariants}

\begin{proposition}[Fractional verifier evidence as a power posterior]
\label{prop:power-posterior}
Let $p_{u,z}$ have prior
$\operatorname{Beta}(\alpha_{0,u},\beta_{0,u})$ with
$\alpha_{0,u},\beta_{0,u}>0$. For binary verifier records
$y_e\in\{0,1\}$ with confidence weights $c_e\in[0,1]$, define the generalized
likelihood
\begin{equation}
  L(p)=\prod_{e\in\mathcal E^{\mathrm{ver}}_{u,z}}
        [p^{y_e}(1-p)^{1-y_e}]^{c_e}.
  \label{eq:power-likelihood}
\end{equation}
Then the generalized posterior \citep{bissiri2016general} is
$\operatorname{Beta}(\alpha^{\mathrm{ver}}_{u,z},
\beta^{\mathrm{ver}}_{u,z})$ with the confidence-weighted update of the gate posterior in \cref{eq:alpha,eq:beta}.
\end{proposition}
\begin{proof}
Multiplying the Beta prior density by \cref{eq:power-likelihood} collects the
powers of $p$ and $(1-p)$ as
$p^{\alpha_{0,u}-1+\sum_{e\in\mathcal E^{\mathrm{ver}}_{u,z}} c_ey_e}$ and
$(1-p)^{\beta_{0,u}-1+\sum_{e\in\mathcal E^{\mathrm{ver}}_{u,z}} c_e(1-y_e)}$. The normalizing constant is the Beta
function, giving the stated generalized posterior. Because fractional powers are
used, this is a power-posterior interpretation of the update in \cref{eq:alpha,eq:beta}.
\end{proof}

\begin{lemma}[Beta density-ratio ordering]
\label{lem:beta-density-ratio}
Let $P_{\alpha,\beta}$ denote a $\operatorname{Beta}(\alpha,\beta)$ random
variable, where $\alpha,\beta>0$. For every $c\ge0$,
\begin{equation}
  P_{\alpha+c,\beta}\succeq_{\mathrm{FOSD}}P_{\alpha,\beta},
  \qquad
  P_{\alpha,\beta+c}\preceq_{\mathrm{FOSD}}P_{\alpha,\beta},
  \label{eq:beta-fosd}
\end{equation}
where $\succeq_{\mathrm{FOSD}}$ denotes first-order stochastic dominance
\citep{shaked2007stochastic}.
The inequalities are strict on the interior $(0,1)$ when $c>0$.
\end{lemma}
\begin{proof}
The case $c=0$ is equality, so take $c>0$. The two Beta densities satisfy
\begin{equation}
  \frac{f_{\alpha+c,\beta}(p)}{f_{\alpha,\beta}(p)}
  =\frac{B(\alpha,\beta)}{B(\alpha+c,\beta)}p^c
  \eqqcolon r_{\alpha,c}(p).
  \label{eq:beta-alpha-ratio}
\end{equation}
The ratio is continuous and strictly increasing. Since both densities integrate
to one, it has a unique crossing $p_\star\in(0,1)$: the density difference is
negative before $p_\star$ and positive after it. For
$D(t)=F_{\alpha+c,\beta}(t)-F_{\alpha,\beta}(t)$, this gives $D(t)\le0$
for $t\le p_\star$. For $t>p_\star$, using $\int_0^1(f_{\alpha+c,\beta}-f_{\alpha,\beta})=0$ gives
\begin{equation}
  D(t)=-\int_t^1(f_{\alpha+c,\beta}(p)-f_{\alpha,\beta}(p))\,dp\le0.
\end{equation}
Thus the first CDF is everywhere no larger, proving the first stochastic-order
relation. Similarly,
\begin{equation}
  \frac{f_{\alpha,\beta+c}(p)}{f_{\alpha,\beta}(p)}
  =\frac{B(\alpha,\beta)}{B(\alpha,\beta+c)}(1-p)^c
\end{equation}
 is strictly decreasing; reversing the preceding sign argument proves the
second relation. Positivity of Beta densities on $(0,1)$ gives strict interior
inequalities for $c>0$.
\end{proof}

\begin{corollary}[Monotonicity of verifier-posterior quantiles]
\label{cor:beta-quantile-detailed}
For $a\in(0,1)$, define
\begin{equation}
  Q_a(\alpha,\beta)=\inf\{p\in[0,1]:F_{\alpha,\beta}(p)\ge a\}.
  \label{eq:beta-quantile-definition}
\end{equation}
Then, for every $c\ge0$,
\begin{equation}
  Q_a(\alpha+c,\beta)\ge Q_a(\alpha,\beta),
  \qquad
  Q_a(\alpha,\beta+c)\le Q_a(\alpha,\beta).
\end{equation}
Consequently, adding verifier success mass $c$ cannot decrease a lower posterior
quantile, and adding verifier failure mass $c$ cannot increase an upper
posterior quantile, provided the other shape parameter and prior hyperparameters
are held fixed.
\end{corollary}
\begin{proof}
If $F_1(t)\le F_0(t)$ for every $t$, the level $a$ is reached by $F_1$ no
earlier than by $F_0$, so $Q_a(F_1)\ge Q_a(F_0)$. Applying this implication
to the CDF orderings in \cref{lem:beta-density-ratio} proves the inequalities.
Under the power-posterior update of \cref{prop:power-posterior}, a verifier success
of mass $c$ changes $(\alpha,\beta)$ to $(\alpha+c,\beta)$ and a verifier failure of mass $c$ changes it to
$(\alpha,\beta+c)$.
\end{proof}

\section{Full experimental results}
\label{app:full-results}

This appendix holds the protocol details and the per-row analysis behind \Cref{sec:setup}.

\paragraph{Benchmarks and metrics.}
The IID datasets are HotpotQA \citep{yang2018hotpotqa}, TriviaQA \citep{joshi2017triviaqa}, AIME 2026, HealthBench \citep{arora2025healthbench}, MBPP+ \citep{austin2021program,liu2023your} and ALFWorld \citep{shridhar2020alfworld}; the OOD datasets are MuSiQue-Ans \citep{trivedi2022musique}, NQ-Open \citep{kwiatkowski2019natural,lee2019latent}, MATH-Hard \citep{hendrycks2021measuring}, GPQA \citep{rein2023gpqa}, SWE-bench \citep{jimenez2024swe} and WebShop \citep{yao2022webshop}.
Every source is pinned to a repository revision and an evaluation population, and each dataset is a fixed 128-record draw under one fixed seed; AIME 2026 is evaluated on all thirty released problems.
HotpotQA, TriviaQA, MuSiQue-Ans and NQ-Open report exact match, the primary metric used in every table and figure, and answer F1.
AIME 2026, MATH-Hard and GPQA report accuracy, HealthBench the native rubric mean, ALFWorld and WebShop the native success rate, SWE-bench the resolved rate, and MBPP+ pass@1 against base and extended tests.
Every primary metric and its verifier version are fixed before any training or evaluation run.

\paragraph{Implementation.}
The Supervisor and the frozen executor are both Qwen3.5-9B.
The phase trigger and paired validation of Appendix~C use a separate validation query set, disjoint from the training queries and from every evaluation record; the six OOD datasets serve only for final evaluation.
All methods share the same executor, tool set, retrieval corpus and inference budget.
Verifier labels come from unit tests, schema validation, assertions and execution status (Appendix~C).

\paragraph{Extended analysis.}
\emph{Main results.} Training the backbone raises the IID exact-match average by at most 7.26 points (GRPO$^\dagger$) but lowers AIME 2026 from 51.33 to 32.00 (SFT) and 36.00 (GRPO$^\dagger$).
Out of distribution, the margin of \RtwoFlow over the best baseline is widest on MuSiQue-Ans (3.90), NQ-Open and WebShop (3.75 each), and its gain over Qwen3.5 ranges from 7.5 points on MATH-Hard to 63.0 on NQ-Open.
With the 9B backbone, \RtwoFlow exceeds DeepSeek-V4-Flash prompted directly on all twelve datasets, including SWE-bench (44.22 against 34.84).
\emph{Executor transfer.} On the six IID datasets the frozen averages range from 62.3 (glm-5.3) to 76.8 (gemini-3.8-flash) and those under \RtwoFlow from 84.7 to 91.6; for every backbone the average gain on the OOD datasets exceeds that on the IID datasets.
Aggregated by OOD domain, the gain is largest on question answering (25.3 to 45.0 points) and interactive tasks (24.1 to 35.6) and smallest on Math/Sci (8.2 to 15.9), where the frozen scores already lie between 76.6 and 88.0; code stays the lowest domain after orchestration (44.8 to 53.0) while still gaining 8.8 to 19.8 points.
\emph{Ablations.} The ablated cells of \Cref{tab:ablation} fall 0.31 to 14.00 points below the full model.
Returning states to full histories drops most on AIME 2026 (14.00), WebShop (10.15) and HealthBench (8.14).
Reading the shares off the learned flow instead of the reference readout costs 2.63 IID and 1.61 OOD points, and ranking by share alone costs most on WebShop (6.72).
Replacing verifier labels by reward-derived labels costs most on WebShop (9.37) and MuSiQue-Ans (6.40); the posterior mean in place of the credible bounds costs 2.23 IID and 2.97 OOD points, skipping the paired non-inferiority check 2.94 and 2.42, and a fixed phase length 1.58 and 2.03.
\emph{Training objectives.} The IID mean of \RtwoTB is 2.25 points above SubTB, 2.98 above TTB, 3.07 above TB, 5.13 above DB and 7.15 above GRPO, which is the lowest objective on every IID dataset. SubTB is within 1.09 points of \RtwoTB on TriviaQA, MBPP+ and ALFWorld but 6.67 behind on AIME 2026.
Per step, \RtwoTB uses 433 learner GPU-seconds, 11 fewer than SubTB and 44 more than TB, at nearly the same learner token volume as SubTB (3,121,184 against 3,121,927); a whole step takes 2,587 GPU-seconds and 832\,s of wall-clock, 7.5\% and 5.6\% more than GRPO.
\emph{Phases, edits and harmful skills.} With fixed-length phases OOD accuracy ends at 78.99; with reward-derived labels it peaks at 77.9 after phase 3; with the library frozen it stays between 73.4 and 74.8 after the first phase.
Of the 44 edits committed on reward-derived labels, 25 raise the training reward while lowering the held-out verified score.
Of the twelve injected harmful skills, \RtwoFlow prunes six by phase 2, three in phases 3--5 and two in phases 6--8, and keeps one.

\paragraph{Seeds and variance.}
Every reported number is a mean over five runs; \Cref{tab:main} also reports the standard deviation, and the per-run results are released with the code.

\end{document}